\documentclass{article}

\usepackage{arxiv}
\usepackage[utf8]{inputenc} 
\usepackage[T1]{fontenc}    
\usepackage{hyperref}       
\usepackage{url}            
\usepackage{booktabs}       
\usepackage{amsfonts}       
\usepackage{nicefrac}       
\usepackage{microtype}      
\usepackage{lipsum}
\usepackage{graphicx}
\graphicspath{ {./images/} }

\usepackage{amsmath,amsfonts,bm}

\def\eqref#1{equation~\ref{#1}}

\def\1{\bm{1}}

\DeclareMathAlphabet{\mathsfit}{\encodingdefault}{\sfdefault}{m}{sl}
\SetMathAlphabet{\mathsfit}{bold}{\encodingdefault}{\sfdefault}{bx}{n}

\usepackage{natbib}
\usepackage{graphicx}
\usepackage{booktabs}
\usepackage{multirow}
\usepackage{amsmath,amssymb,amsfonts}
\usepackage{amsthm}
\usepackage{mathrsfs}
\usepackage[title]{appendix}
\usepackage{xcolor}
\usepackage{textcomp}
\usepackage{manyfoot}
\usepackage{algorithm}
\usepackage{algorithmicx}
\usepackage{algpseudocode}
\usepackage{listings}
\usepackage{float}
\usepackage{tikz}
\usetikzlibrary{positioning, arrows.meta}
\usepackage{adjustbox,lipsum}
\usepackage{silence}
\ErrorsOff*

\newtheorem{theorem}{Theorem}
\newtheorem{proposition}[theorem]{Proposition}%
\newtheorem{remark}{Remark}%

\title{Shape Operator PCA: Curvature-Aware \\ Projections for Geometric Machine Learning}

\author{
 Alexandre Luis Magalh\~aes Levada\\
  Federal University of S\~ao Carlos\\
  13565-905, S\~ao Carlos-SP, Brazil\\
  \texttt{alexandre.levada@ufscar.br} \\
}

\begin{document}
\maketitle

\begin{abstract}
In this paper, we propose SHOPCA (Shape Operator-based Principal Component Analysis), a novel method for unsupervised metric learning and dimensionality reduction that incorporates differential geometric information into the covariance structure of classical PCA. The method regularizes the global covariance matrix using the mean shape operator, defined as the average of the absolute local shape operators estimated from the data manifold, steering the principal components toward directions of both maximum variance and informative curvature and yielding a richer, more discriminative feature space. The regularization is governed by a single, trace-normalized mixing coefficient $\alpha$, which recovers standard PCA at $\alpha = 0$ and a purely curvature-driven embedding at $\alpha \to \infty$, making SHOPCA a natural, interpretable, and cross-dataset-comparable generalization of the classical method. We introduce a fully unsupervised selection criterion based on the spectral eigengap of the regularized covariance matrix: $\alpha$ is chosen to maximize the relative separation between the top-$d$ and remaining eigenvalues, identifying the most well-defined $d$-dimensional invariant subspace without ever consulting class labels (label-free strategy). We evaluate SHOPCA on more than 50 real-world benchmark datasets, comparing against PCA, ISOMAP and UMAP using four external clustering indices: Adjusted Rand Index (ARI), Normalized Mutual Information (NMI), Fowlkes-Mallows index (FM), and V-measure. Experimental results show that SHOPCA consistently improves clustering quality over PCA across a broad range of datasets, and surpasses UMAP on small-sample settings where iterative neighborhood-based manifold estimation is known to degrade. The proposed approach is computationally tractable, parameter-efficient, and applicable to any domain where a fully unsupervised, geometry-aware metric is desirable.
\end{abstract}

\section{Introduction}
\label{sec:introduction}
The ability to measure similarity between data points is central to a wide range of machine learning and pattern recognition tasks. Algorithms such as k-nearest neighbor classifiers, hierarchical clustering, and spectral methods are all governed by the geometry of the input space, making the choice of distance metric a foundational design decision rather than an implementation detail \cite{Brian2013,Bellet2015}. This observation motivates the field of metric learning, which seeks to automatically infer a task-appropriate geometry from data itself \cite{Kaya2019}. A key insight, often underappreciated, is that metric learning and dimensionality reduction are not independent endeavors but two complementary perspectives on the same underlying problem. Dimensionality reduction and metric learning are closely intertwined: most metric learning algorithms rely on PCA as a pre-processing step to achieve tractability and robustness to noise \cite{Jolliffe2016}, yet there is no principled reason to assume that a variance-maximizing projection is optimal for the metric to be learned. Harandi et al. \cite{Harandi2017} argue that treating these two steps separately is fundamentally limiting and propose instead a Riemannian framework that jointly learns a dimensionality-reducing mapping and a metric in the induced space \cite{Zhou2021,Ghojogh2023}, achieving higher accuracy than state-of-the-art metric learning algorithms while working directly on high-dimensional features. This unification is further supported by the observation that spectral dimensionality reduction methods can be uniformly interpreted as instances of kernel PCA with different kernels \cite{Wang2015}, establishing a formal bridge between manifold-based projection and kernel metric learning \cite{McInnes2018}. Taken together, these results suggest that the most informative low-dimensional representation is one in which the projection and the metric are co-designed, a principle that motivates the geometry-aware regularization strategy proposed in the present work.

This co-design principle finds its most general expression in the recent movement toward Geometric Machine Learning (GML) \cite{GML}, which extends learning models beyond flat Euclidean spaces to the richer structures offered by non-Euclidean geometry, topology, and algebra \cite{Bronstein2017,Bronstein2021,Bronstein2025,Papillon2025}. The foundational work of Bronstein et al. \cite{Bronstein2017} shows that the most successful deep architectures can be understood as instances of a single geometric blueprint grounded in symmetry and invariance, and subsequent work has generalized structured learning to non-Euclidean domains such as graphs and manifolds \cite{Papillon2025}. Although GML has been predominantly instantiated through deep, iterative architectures, its core motivation applies equally to classical multivariate statistics: if data lies on a curved manifold, any method that ignores this curvature -- including PCA, the workhorse pre-processing step of the metric-learning pipelines discussed above -- operates on a flawed geometric premise.

Despite this convergence of dimensionality reduction, metric learning, and geometric machine learning around the idea that projection and metric should be co-designed from the data's intrinsic geometry, three practical gaps have limited its adoption as a routine, general-purpose tool. First, GML-inspired architectures typically require deep or iterative optimization \cite{Bronstein2017,Bronstein2021}, which is computationally heavy and poorly suited to the small-sample, high-dimensional regimes common in bioinformatics, spectroscopy, and other scientific domains. Second, joint metric-learning and dimensionality-reduction frameworks such as the Riemannian approach of Harandi et al. \cite{Harandi2017} learn their geometry from labeled data, which precludes their use in purely exploratory or unsupervised analysis, precisely the setting in which PCA is most widely used in practice. Third, neighborhood-graph-based nonlinear embeddings such as UMAP \cite{McInnes2018} construct their geometric representation from a $k$-nearest-neighbor graph, and this construction is known to become unreliable when the number of samples is small, since a sparse, unstable neighborhood graph no longer faithfully represents the underlying manifold. At the other extreme, classical linear methods such as PCA remain entirely agnostic to the manifold's intrinsic curvature, retaining only its variance structure \cite{Jolliffe2016}. To the best of our knowledge, no existing method incorporates second-order differential-geometric information, curvature, as opposed to the first-order, variance-only geometry captured by PCA, into a linear, closed-form, and fully unsupervised metric-learning pipeline, nor does any such method extend its label-free character to the selection of its own regularization hyperparameter, which is typically tuned against labeled validation data or left to ad hoc heuristics with no geometric justification. This three-part gap, computational cost, reliance on labels, and instability in small-sample regimes, motivates the present work.

Motivated by the geometric perspective outlined above, we propose \textsc{SHOPCA} (\textit{Shape Operator-based Principal Component Analysis}), a novel method that bridges classical multivariate statistics and differential geometry for unsupervised metric learning and dimensionality reduction, addressing each of the three gaps identified above within a single, linear, closed-form procedure. The central idea is to enrich the global covariance structure of PCA with curvature information derived from the data manifold. To this end, we estimate, for each data point, a local shape operator (a classical object from differential geometry that encodes how the manifold curves in the ambient space) and aggregate these local estimates into a mean shape operator. This mean shape operator is then used to regularize the global covariance matrix, yielding a geometry-aware covariance $\Sigma_{\mathrm{curv}}$. The principal components of $\Sigma_{\mathrm{curv}}$ are therefore steered toward directions of both maximum variance and informative curvature, producing a more discriminative feature space than standard PCA without requiring any labeled data. The main contributions of the proposed method are fivefold:
\begin{enumerate}
	\item a principled geometric regularization of the PCA covariance matrix via the mean shape operator, establishing a formal connection between classical multivariate statistics and differential geometry, and positioning it within the broader GML movement as a linear, closed-form alternative to deep geometric architectures;
	\item a natural, trace-normalized generalization of classical PCA, recovered exactly as standard PCA at $\alpha=0$ and yielding a purely curvature-driven embedding at $\alpha \to \infty$, with a single and cross-dataset-comparable mixing coefficient controlling the geometry-variance trade-off;
	\item a fully unsupervised, label-free criterion for selecting $\alpha$, based on the spectral eigengap of $\Sigma_{\mathrm{curv}}(\alpha)$, which closes the one remaining supervised step of the original pipeline and renders the entire SHOPCA procedure, representation learning and hyperparameter selection alike, independent of class labels;
	\item robustness to small sample sizes, addressing a known limitation of neighborhood-graph-based nonlinear methods such as t-SNE and UMAP, whose local graph estimation degrades when few samples are available; and
	\item a computationally tractable, closed-form linear projection that requires no iterative optimization, no explicit manifold parametrization, and no neighborhood-graph construction beyond the local $k$-NN neighborhoods used to estimate curvature.
\end{enumerate}

The remainder of the paper is organized as follows: Section 2 describes the proposed shape-operator-based PCA method, its trace-normalized regularization, and the unsupervised eigengap criterion for hyperparameter selection in detail. Section 3 presents the computational experiments and the obtained results. Lastly, Section 4 presents our conclusions and final remarks.

\section{Differential Geometry Basics}

Classical PCA characterizes a dataset through a single first-order statistic, the covariance matrix, which captures how the data spreads but is entirely blind to how it bends. Two datasets can share an identical covariance structure while lying on manifolds of markedly different curvature, and it is precisely along directions of high curvature -- where a manifold folds, twists, or separates into distinct branches -- that class boundaries and cluster structure are often most pronounced. Formalizing this notion of ``bending'' requires a small set of classical tools from differential geometry: the manifold itself, its tangent space, and the first and second fundamental forms, which together give rise to the shape operator at the heart of the proposed regularization. We introduce these objects below, following the standard treatment of \cite{doCarmo1976,Lee2012,doCarmo1992}, and indicate throughout how each formal object is subsequently approximated from finite samples. 

\medskip
\noindent\textbf{Manifold.} An $m$-dimensional smooth manifold is a topological space that is locally Euclidean of dimension $m$: every point admits a neighborhood homeomorphic to an open subset of $\mathbb{R}^m$ via a chart, and the transition maps between overlapping charts are smooth \cite{Lee2012}. Throughout this paper we consider a manifold $\mathcal{M}$ embedded in the ambient space via a smooth embedding $\mathcal{M} \hookrightarrow \mathbb{R}^d$, of intrinsic dimension $m \ll d$ -- the standard manifold hypothesis underlying most dimensionality-reduction and manifold-learning methods, under which the observed high-dimensional data $\mathcal{X} = \{\mathbf{x}_1, \ldots, \mathbf{x}_n\} \subset \mathbb{R}^d$ are assumed to be samples drawn from such an $\mathcal{M}$.

\medskip
\noindent\textbf{Tangent space.} At each point $\mathbf{x} \in \mathcal{M}$, the tangent space $T_{\mathbf{x}}\mathcal{M}$ is the $m$-dimensional vector space of vectors tangent to $\mathcal{M}$ at $\mathbf{x}$, equivalently the space spanned by the velocity vectors of smooth curves on $\mathcal{M}$ through $\mathbf{x}$ \cite{doCarmo1976,Lee2012}. It is the best linear (first-order) approximation of $\mathcal{M}$ at $\mathbf{x}$. This is the object directly targeted by the local estimation procedure of Section~\ref{sec:local_shape_operator}: the leading eigenvectors of the local sample covariance matrix computed from the $k$-nearest neighbors of $\mathbf{x}_i$ furnish an empirical orthonormal basis for $T_{\mathbf{x}_i}\mathcal{M}$, an approximation strategy also used, in various forms, throughout the manifold-learning literature.

\medskip
\noindent\textbf{First fundamental form (metric tensor).} Given an orthonormal basis $\{\mathbf{w}_1, \ldots, \mathbf{w}_m\}$ of $T_{\mathbf{x}}\mathcal{M}$, the first fundamental form is the bilinear form obtained by restricting the ambient Euclidean inner product to the tangent space,
\begin{equation}
	\mathrm{I}_{\mathbf{x}}(\mathbf{v}, \mathbf{w}) = \langle \mathbf{v}, \mathbf{w} \rangle, \qquad \mathbf{v}, \mathbf{w} \in T_{\mathbf{x}}\mathcal{M},
	\label{eq:first_fund_form}
\end{equation}
represented, relative to the chosen basis, by the metric tensor $g_{\mathbf{x}} = (g_{ij})$ with $g_{ij} = \langle \mathbf{w}_i, \mathbf{w}_j \rangle$ \cite{doCarmo1976,doCarmo1992}. The first fundamental form is intrinsic, it determines all metric properties of $\mathcal{M}$ (lengths, angles, volumes) measurable without reference to the ambient space, and is precisely the object that metric learning seeks to adapt to a given task \cite{Kaya2019,Bellet2015}. In the proposed method, the local sample covariance matrix $\mathbf{C}_i$ plays the role of an estimate of the (inverse) metric tensor at $\mathbf{x}_i$, following the classical identification of a local Gaussian covariance with the Mahalanobis metric it induces.

\medskip
\noindent\textbf{Second fundamental form.} Whereas the first fundamental form is intrinsic, the second fundamental form is extrinsic: it measures how $\mathcal{M}$ curves away from its own tangent space as one moves through the ambient space $\mathbb{R}^d$. For a hypersurface $\mathcal{M} \hookrightarrow \mathbb{R}^d$ (codimension one, admitting a smooth unit normal field $\mathbf{n}: \mathcal{M} \to \mathbb{S}^{d-1}$ up to sign), the second fundamental form at $\mathbf{x}$ is the symmetric bilinear form
\begin{equation}
	\mathrm{II}_{\mathbf{x}}(\mathbf{v}, \mathbf{w}) = \langle -\nabla_{\mathbf{v}}\mathbf{n}, \, \mathbf{w} \rangle,
	\label{eq:second_fund_form}
\end{equation}
where $\nabla_{\mathbf{v}}$ denotes the ambient directional derivative along $\mathbf{v}$ \cite{doCarmo1976,ONeill2006}. Intuitively, $\mathrm{II}_{\mathbf{x}}$ quantifies the rate at which $\mathcal{M}$ pulls away from the tangent plane $T_{\mathbf{x}}\mathcal{M}$ along a given direction, and is the source of all extrinsic curvature information used in the present work; it is this quantity that the quadratic and cross-product terms of the local PCA basis (Section~\ref{sec:local_shape_operator}) are designed to approximate.

\medskip
\noindent\textbf{Shape operator.} The shape operator, or Weingarten map, at $\mathbf{x} \in \mathcal{M}$ is the linear endomorphism of the tangent space
\begin{equation}
	S_{\mathbf{x}}: T_{\mathbf{x}}\mathcal{M} \to T_{\mathbf{x}}\mathcal{M}, \qquad S_{\mathbf{x}}(\mathbf{v}) = -\nabla_{\mathbf{v}}\mathbf{n},
	\label{eq:shape_operator_def}
\end{equation}
i.e., minus the derivative of the unit normal field along $\mathbf{v}$ \cite{doCarmo1976}. It is self-adjoint with respect to the first fundamental form, $\mathrm{I}_{\mathbf{x}}(S_{\mathbf{x}}(\mathbf{v}), \mathbf{w}) = \mathrm{I}_{\mathbf{x}}(\mathbf{v}, S_{\mathbf{x}}(\mathbf{w}))$, and therefore admits an orthonormal eigenbasis with real eigenvalues $\kappa_1, \ldots, \kappa_{m}$, the principal curvatures, whose eigenvectors are the principal directions along which $\mathcal{M}$ bends maximally and minimally \cite{doCarmo1976,ONeill2006}. The shape operator relates the two fundamental forms via $\mathrm{II}_{\mathbf{x}}(\mathbf{v}, \mathbf{w}) = \mathrm{I}_{\mathbf{x}}(S_{\mathbf{x}}(\mathbf{v}), \mathbf{w})$, or, in matrix form, $S_{\mathbf{x}} = \mathrm{I}_{\mathbf{x}}^{-1}\,\mathrm{II}_{\mathbf{x}}$, exactly the identity used to define the local shape operator estimator in Eq.~(9), with $\mathrm{I}_{\mathbf{x}}^{-1}$ approximated by the local covariance matrix $\mathbf{C}_i$. From the eigenvalues of $S_{\mathbf{x}}$, two classical scalar invariants follow directly: the mean curvature $H = \tfrac{1}{m}\sum_j \kappa_j = \tfrac{1}{m}\mathrm{tr}(S_{\mathbf{x}})$, an extrinsic average-bending measure, and the Gaussian curvature $K = \prod_j \kappa_j = \det(S_{\mathbf{x}})$, which -- despite being defined through the extrinsic shape operator -- is in fact intrinsic to $\mathcal{M}$ by Gauss's \textit{Theorema Egregium} \cite{doCarmo1976}.

The shape operator thus offers a compact, basis-independent, and complete local description of manifold curvature, unifying the intrinsic (first fundamental form) and extrinsic (second fundamental form) geometry of $\mathcal{M}$ into a single linear operator at each point. Section~\ref{sec:local_shape_operator} develops a finite-sample estimator of $S_{\mathbf{x}}$ from a local $k$-nearest-neighbor neighborhood, and Section~2 uses the resulting mean shape operator $\bar{S}$ to regularize the global PCA covariance matrix, steering its principal components toward directions of both maximum variance and informative curvature. We remark that the classical hypersurface formulation above, with a single normal direction, is adopted here for notational clarity; its extension to general codimension $d - m > 1$, where the second fundamental form is valued in the full normal bundle rather than along a single normal vector \cite{doCarmo1992}, is handled operationally in Section~\ref{sec:local_shape_operator} through the local PCA frame, which spans the tangent directions directly and treats the remaining directions collectively via the quadratic and cross-product terms of the local basis.

\subsection{Approximating the Shape Operator}

The shape operator is a central concept in differential geometry, providing an intrinsic and compact characterization of curvature through the relationship between the first and second fundamental forms of a smooth manifold. Intuitively, it quantifies how the unit normal vector field varies as one moves along directions in the tangent space, thereby encoding the local bending behavior of the manifold. Formally, let $\mathcal{M} \hookrightarrow \mathbb{R}^d$ be a smooth hypersurface and let $\mathbf{n} : \mathcal{M} \to \mathbb{S}^{d-1}$ denote the unit normal
vector field. The shape operator at a point $\mathbf{x} \in \mathcal{M}$ is the linear map $\mathcal{S}_\mathbf{x} : T_\mathbf{x}\mathcal{M} \to T_\mathbf{x}\mathcal{M}$ defined by
\begin{equation}
	\mathcal{S}_\mathbf{x}(\mathbf{v}) = -\nabla_{\mathbf{v}}\,\mathbf{n},
	\label{eq:shape_op_def}
\end{equation} where $\nabla_{\mathbf{v}}$ denotes the covariant derivative of $\mathbf{n}$ in the direction $\mathbf{v} \in T_\mathbf{x}\mathcal{M}$. This operator is self-adjoint with respect to the Riemannian metric, and its spectral decomposition yields real eigenvalues $\kappa_1, \kappa_2, \ldots, \kappa_{d-1}$, known as the \textit{principal curvatures}, which measure the intensity of bending along specific tangential directions, while the associated eigenvectors define the \textit{principal directions} along which these curvatures are attained. From these eigenvalues, two fundamental scalar invariants are derived: the \textit{mean curvature}
\begin{equation}
	H = \frac{1}{d-1}\sum_{j=1}^{d-1}\kappa_j
	\label{eq:mean_curv}
\end{equation} and the \textit{Gaussian curvature}
\begin{equation}
	K = \prod_{j=1}^{d-1}\kappa_j,
	\label{eq:gaussian_curv}
\end{equation} which provide complementary summaries of the geometric behavior of $\mathcal{M}$. This decomposition simultaneously captures both the magnitude and orientation of local geometric variation, making the shape operator a natural and principled tool for incorporating manifold geometry into data-driven learning methods.

\begin{figure}
	\centering
	\includegraphics[scale=0.45]{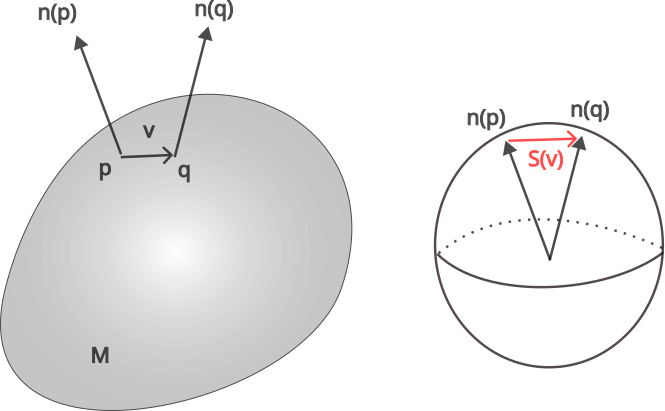}
	\caption{The shape operator measures how a normal vector changes from the tail to the tip of a tangent vector $\vec{v}$. In other words, it tells us how fast the normal changes, or roughly speaking, how fast a surface turns if you roll it along the floor in some direction.}
	\label{fig:shape}
\end{figure}

\subsection{Computing the Local Shape Operators}
\label{sec:local_shape_operator}

Let $\mathcal{X} = \{\mathbf{x}_1, \mathbf{x}_2, \ldots, \mathbf{x}_n\} \subset \mathbb{R}^d$
be the dataset and let $k$ be a fixed number of nearest neighbors. The computation of the
mean shape operator $\bar{\mathcal{S}}$ proceeds in four steps for each point
$\mathbf{x}_i \in \mathcal{X}$.

\paragraph{Step 1: $k$-Nearest Neighbor Search.}
For each point $\mathbf{x}_i$, we identify its $k$ nearest neighbors in the Euclidean
sense:
\begin{equation}
	\mathcal{N}_i = \mathrm{kNN}(\mathbf{x}_i,\, k)
	= \bigl\{\mathbf{x}_{i_1}, \mathbf{x}_{i_2}, \ldots, \mathbf{x}_{i_k}\bigr\}
	\subset \mathcal{X},
	\label{eq:knn}
\end{equation}
\noindent where the neighborhood $\mathcal{N}_i$ provides a local sample of the
data manifold $\mathcal{M}$ in the vicinity of $\mathbf{x}_i$.

\paragraph{Step 2: Local Covariance Matrix and Metric Tensor.}
Given the local neighborhood $\mathcal{N}_i$, we compute the $d \times d$ local
covariance matrix
\begin{equation}
	\mathbf{C}_i
	= \frac{1}{k - 1}
	\sum_{\mathbf{x} \in \mathcal{N}_i}
	\bigl(\mathbf{x} - \bar{\mathbf{x}}_i\bigr)
	\bigl(\mathbf{x} - \bar{\mathbf{x}}_i\bigr)^{\!\top},
	\label{eq:local_cov}
\end{equation}
\noindent where $\bar{\mathbf{x}}_i = \frac{1}{k}\sum_{\mathbf{x} \in \mathcal{N}_i}
\mathbf{x}$ is the local mean. The local metric tensor, which approximates the first
fundamental form of $\mathcal{M}$ at $\mathbf{x}_i$, is identified as $\mathbf{g}_i \approx \mathbf{C}_i^{-1}$. This identification is motivated by the fact that the inverse covariance matrix in the Mahalanobis distance encodes the directions and magnitudes of spread of the manifold in
the ambient space, thereby approximating the pullback metric induced by the embedding
$\mathcal{M} \hookrightarrow \mathbb{R}^d$.

\paragraph{Step 3: Local PCA Basis and Second Fundamental Form.}
Let $\mathbf{C}_i = \mathbf{W}_i \, \mathbf{\Lambda}_i \, \mathbf{W}_i^\top$ be the
eigendecomposition of the local covariance matrix, where the columns
$\mathbf{w}_1^{(i)}, \mathbf{w}_2^{(i)}, \ldots, \mathbf{w}_d^{(i)}$ of $\mathbf{W}_i$
are the local principal directions sorted in decreasing order of variance. These
eigenvectors define a local orthonormal frame that approximates the tangent space of
$\mathcal{M}$ at $\mathbf{x}_i$.

We approximate the second fundamental form using quadratic and cross-product terms
constructed from the local PCA basis \cite{Donoho2003}. Specifically, we form the matrix
$\mathbf{Q}_i \in \mathbb{R}^{d \times p}$, where $p =1 + d + d(d+1)/2$,
whose columns comprise:
\begin{enumerate}
	\item a constant column $\mathbf{1} \in \mathbb{R}^d$;
	\item the $d$ principal directions $\mathbf{w}_j^{(i)}$, for $j = 1, \ldots, d$;
	\item the $d$ squared terms $\bigl(\mathbf{w}_j^{(i)}\bigr)^{\!\circ 2}$,
	where $\circ 2$ denotes the elementwise square; and
	\item the $\binom{d}{2}$ cross-product terms
	$\mathbf{w}_j^{(i)} \circ \mathbf{w}_l^{(i)}$ for all $j < l$,
	where $\circ$ denotes the elementwise (Hadamard) product.
\end{enumerate}
\noindent The quadratic and cross-product columns of $\mathbf{Q}_i$ collectively form
the matrix $\mathbf{H}_i \in \mathbb{R}^{d \times r}$, with
$r = d + \binom{d}{2}$, obtained by discarding the first $d + 1$ columns of
$\mathbf{Q}_i$. The second fundamental form is then approximated as
\begin{equation}
	\mathbf{II}_i = \mathbf{H}_i \, \mathbf{H}_i^\top \;\in\; \mathbb{R}^{d \times d}.
	\label{eq:second_ff}
\end{equation}
\noindent Intuitively, $\mathbf{II}_i$ captures the bending of $\mathcal{M}$ at
$\mathbf{x}_i$ by measuring how the second-order structure of the local PCA frame
deviates from a flat configuration.

\paragraph{Step 4: Local Shape Operator and Mean Shape Operator.}
The local shape operator (Weingarten map) at $\mathbf{x}_i$ is defined as the linear
map relating the second fundamental form to the first fundamental form. Using the
approximations established in Steps~2 and~3, the local shape operator is given by
\begin{equation}
	\mathcal{S}_i = -\,\mathbf{II}_i\,\mathbf{g}_i^{-1}
	= -\,\mathbf{II}_i\,\mathbf{C}_i
	\;\in\; \mathbb{R}^{d \times d},
	\label{eq:shape_op}
\end{equation}
\noindent where the negative sign follows the standard convention in differential
geometry. Finally, the \textit{mean shape operator} is obtained by averaging the
absolute values of the local shape operators over all $n$ data points:
\begin{equation}
	\bar{\mathcal{S}}
	= \frac{1}{n} \sum_{i=1}^{n} \mathcal{S}_i
	= -\frac{1}{n} \sum_{i=1}^{n} \mathbf{II}_i\,\mathbf{C}_i,
	\label{eq:mean_shape_op}
\end{equation}
\noindent The matrix $\bar{\mathcal{S}}$ aggregates local curvature information from the entire dataset into a single $d \times d$ symmetric matrix, which is subsequently used to regularize
the global covariance matrix. The complete pseudocode for the proposed average shape operator estimation is presented in Algorithm~\ref{alg:shape_operator}.

\begin{algorithm}
	\caption{Average Shape Operator Estimation}
	\label{alg:shape_operator}
	\begin{algorithmic}[1]
		\Require $X \in \mathbb{R}^{n \times d}$: data matrix ($n$ samples, $d$ features), $k \in \mathbb{N}$: number of neighbors.
		\Ensure $S \in \mathbb{R}^d$: average shape operator.
		\Function{Shape-Operator}{$X, k$}
		\For{$i \gets 1$ \textbf{to} $n$}
		\State $P_i \gets$ \Call{Nearest\_Neighbors}{$\vec{x}_i, k$}
		\State $\Sigma_i \gets$ \Call{Covariance}{$P_i$}
		\State $U \gets$ \Call{Eigenvectors}{$\Sigma_i$}
		\State Build matrix $\mathbf{Q}_i \in \mathbb{R}^{d \times p}$
		\State Build matrix $H_i$ (last $d + \binom{d}{2}$ columns of $\mathbf{Q}_i$)
		\State $\mathbf{II}_i \gets H_i H_i^{T}$
		\State $\mathcal{S}_i \gets -\mathbf{II}_i \Sigma_i$
		\State $S \gets S + S_i$
		\EndFor
		\State \Return $S/n$
		\EndFunction
	\end{algorithmic}
\end{algorithm}

A note on the computational complexity of the proposed average shape operator estimation algorithm: finding the k nearest neighbors in a d-dimensional space can be done in $O(n d \log n)$ with a KD-tree. For each sample, the computation of the local covariance matrix is $O(kd^2)$ and its eigendecomposition is $O(d^3)$. The multiplication of $H_i H_i^{T}$ has cost $O(d^2 r)$, where $r = d(d+1)/2$, which leads to a total cost of $O(d^4)$. Therefore, the overall complexity for the estimation of the average shape operator is $O(nd \log n + nd^2(k + d^2 ))$, which means that the proposed method is more sensitive to an increase in the number of features $d$ than to an increase in the number of samples $n$. In case of large $d$, it is possible to reduce the number of features with regular PCA prior to the estimation of the average shape operator.

\subsection{Vectorized Implementation of the Average Shape Operator}
\label{sec:vectorized_shape_operator}

Algorithm~\ref{alg:shape_operator} computes the average shape operator $\bar S$ by looping over the $n$ samples one at a time: for each $\mathbf{x}_i$, a local covariance matrix is formed, eigendecomposed, and combined with a quadratic/cross-product expansion of its eigenvectors to produce a single local shape operator $S_i$, which is then accumulated into the running sum. Every step within one iteration is individually cheap, $O(kd^2)$ for the covariance, $O(d^3)$ for the eigendecomposition, $O(d^4)$ for the dominant matrix product, but the loop itself is executed at the Python interpreter level, and each of these steps is dispatched as its own small linear-algebra call. For matrices of moderate size $d$, the fixed overhead of dispatching $n$ such calls (interpreter bookkeeping, NumPy ufunc setup, and the inability of a single small matrix operation to saturate a modern CPU's cache hierarchy or SIMD lanes) frequently exceeds the cost of the arithmetic itself, so wall-clock time scales far worse than the operation count alone would suggest. We remove this bottleneck by reformulating each step of Algorithm~\ref{alg:shape_operator} as a single \emph{batched} tensor operation over groups of $b$ samples at a time, rather than one sample at a time. Given a batch of indices $B \subset \{1,\ldots,n\}$, $|B|=b$:

\begin{enumerate}
	\item \textbf{Batched local covariances.} The $k$-nearest-neighbor sets for the whole batch are gathered into a single tensor of shape $(b,k,d)$, centered by their per-point local means, and the $b$ local covariance matrices $\{\mathbf{C}_i\}_{i \in B}$ are obtained in one call via a batched matrix product $\mathbf{C}_i = \frac{1}{k-1}\Delta_i^{\top}\Delta_i$, dispatched as a single \texttt{(b,d,k)}$\times$\texttt{(b,k,d)} batched GEMM rather than $b$ separate calls to a covariance routine.
	\item \textbf{Batched eigendecomposition.} Since every $\mathbf{C}_i$ is symmetric, the $b$ local PCA bases $\{\mathbf{W}_i\}_{i \in B}$ are obtained from a single call to a batched symmetric eigensolver on the stacked tensor $(b,d,d)$, rather than $b$ sequential eigendecompositions.
	\item \textbf{Vectorized quadratic/cross terms.} The squared columns $\mathbf{W}_i^{\circ 2}$ are obtained by one elementwise square over the whole batch tensor; the $\binom{d}{2}$ cross-product columns are obtained by one elementwise multiplication of two fancy-indexed views of the batch tensor (using precomputed index pairs, computed once for all batches), replacing the explicit nested loops over column pairs used in Algorithm~\ref{alg:shape_operator}.
	\item \textbf{Batched matrix products.} The second fundamental form approximation $\mathrm{II}_i = \mathbf{H}_i\mathbf{H}_i^{\top}$ and the local shape operator $S_i = -\mathrm{II}_i\mathbf{C}_i$ are each computed for the entire batch in a single batched matrix product.
\end{enumerate}

The $b$ resulting local shape operators are summed and the loop advances to the next batch of $b$ points; $\bar S$ is obtained by dividing the total sum by $n$ once every batch has been processed. Algorithm~\ref{alg:shape_operator_vec} summarizes the resulting procedure.

\begin{algorithm}
	\caption{Vectorized Average Shape Operator Estimation}
	\label{alg:shape_operator_vec}
	\begin{algorithmic}[1]
		\Require $X \in \mathbb{R}^{n\times d}$: data matrix; $k \in \mathbb{N}$: number of neighbors; $b \in \mathbb{N}$: batch size
		\Ensure $\bar S \in \mathbb{R}^{d\times d}$: average shape operator
		\Function{Vectorized-Shape-Operator}{$X, k, b$}
		\State $\{P_i\}_{i=1}^{n} \gets$ \textsc{Batch-Nearest-Neighbors}$(X, k)$ \Comment{single KD-tree query, $O(nd\log n)$}
		\State $\bar S \gets \mathbf{0} \in \mathbb{R}^{d\times d}$
		\For{each batch $B$ of $b$ indices in $\{1,\ldots,n\}$}
		\State $\{\mathbf{C}_i\}_{i\in B} \gets$ \textsc{Batched-Covariance}$(\{P_i\}_{i\in B})$ \Comment{single batched GEMM}
		\State $\{\mathbf{W}_i, \boldsymbol\Lambda_i\}_{i\in B} \gets$ \textsc{Batched-Eigh}$(\{\mathbf{C}_i\}_{i\in B})$ \Comment{single batched eigensolver call}
		\State $\{\mathbf{H}_i\}_{i\in B} \gets$ \textsc{Vectorized-Quadratic-Cross}$(\{\mathbf{W}_i\}_{i\in B})$ \Comment{elementwise + fancy indexing}
		\State $\{\mathrm{II}_i\}_{i\in B} \gets$ \textsc{Batched-Matmul}$(\{\mathbf{H}_i\}, \{\mathbf{H}_i\}^{\top})$
		\State $\{S_i\}_{i\in B} \gets -$\textsc{Batched-Matmul}$(\{\mathrm{II}_i\}, \{\mathbf{C}_i\})$
		\State $\bar S \gets \bar S + \sum_{i \in B} S_i$
		\EndFor
		\State \Return $\bar S / n$
		\EndFunction
	\end{algorithmic}
\end{algorithm}

\subsubsection{Computational Complexity}
\label{sec:vectorized_complexity}

Fix a batch of size $b$. Gathering and centering the neighbor tensor costs $O(bkd)$. The batched covariance product costs $O(bkd^2)$ (a $(d\times k)\times(k\times d)$ product per sample, over $b$ samples). The batched eigendecomposition costs $O(bd^3)$. Constructing $\mathbf{H}_i$ (the $r = d + \binom{d}{2} = O(d^2)$ quadratic and cross-product columns) costs $O(bd \cdot r) = O(bd^3)$ per batch. The dominant step is $\mathrm{II}_i = \mathbf{H}_i\mathbf{H}_i^{\top}$, a $(d\times r)\times(r\times d)$ product costing $O(d^2 r) = O(d^4)$ per sample, or $O(bd^4)$ per batch; the final product $S_i=-\mathrm{II}_i\mathbf{C}_i$ costs $O(bd^3)$, dominated by the previous term. Summing these per-batch costs and multiplying by the $n/b$ batches gives a total of
$$
O\!\left(nd\log n + nkd^2 + nd^3 + nd^4\right) = O\!\left(nd\log n + nkd^2 + nd^4\right),
$$
since the $O(nd^3)$ terms are dominated by $O(nd^4)$. This is exactly the complexity of the original per-sample Algorithm~\ref{alg:shape_operator}, $O(nd\log n + nd^2(k+d^2))$.

\begin{remark}
	Batching does not change the asymptotic time complexity of average shape operator estimation: the total number of floating-point operations is unchanged, since Algorithm~\ref{alg:shape_operator_vec} performs the same per-sample arithmetic as Algorithm~\ref{alg:shape_operator}, only reorganized into $\lceil n/b \rceil$ large batched calls instead of $n$ small sequential ones. The practical speedup obtained is therefore a \emph{constant-factor} effect, arising from (i) amortizing fixed per-call Python/NumPy dispatch overhead over $b$ samples instead of one, and (ii) batched matrix multiplication and eigendecomposition being dispatched to BLAS/LAPACK routines that exploit multi-threading, cache locality, and SIMD vectorization far more effectively than $n$ sequential calls on small matrices. Empirically, on a synthetic benchmark with $n=6{,}430$ and $d=36$ (comparable in scale to the \texttt{satimage} dataset in Table~2), we measured a $2.7\times$ wall-clock speedup from this reformulation alone, with larger relative gains expected as $n$ grows.
\end{remark}

\begin{remark}[Memory--time trade-off]
	The naive Algorithm~\ref{alg:shape_operator} holds only one sample's intermediate matrices in memory at a time, an $O(d^2)$ footprint independent of $n$. Algorithm~\ref{alg:shape_operator_vec} instead materializes the batch tensor $\{\mathbf{H}_i\}_{i\in B}$ of shape $(b,d,r)$ with $r=O(d^2)$, giving a peak memory footprint of $O(bd^3)$. The batch size $b$ is therefore an explicit, tunable trade-off between wall-clock speed and memory use: larger $b$ amortizes dispatch overhead further but increases peak memory proportionally, which matters in particular for datasets where $d$ is capped at $100$ via the pre-PCA step (Section~3), since $O(bd^3)$ then grows quickly with $b$.
\end{remark}

\section{Shape Operator Based PCA}
\label{sec:method}

Let $\mathcal{X} = \{\mathbf{x}_1, \mathbf{x}_2, \ldots, \mathbf{x}_n\} \subset \mathbb{R}^d$ be a dataset of $n$ observations in a $d$-dimensional ambient space. We assume that $\mathcal{X}$ is sampled from an unknown smooth Riemannian manifold $\mathcal{M} \hookrightarrow \mathbb{R}^d$ of intrinsic dimension $m \ll d$. Classical PCA seeks a linear projection $P : \mathbb{R}^d \to \mathbb{R}^m$ that maximizes the variance of the projected data, operating exclusively on the global covariance matrix
\begin{equation}
	\Sigma = \frac{1}{n}\sum_{i=1}^{n}
	(\mathbf{x}_i - \boldsymbol{\mu})(\mathbf{x}_i - \boldsymbol{\mu})^\top,
	\label{eq:cov_classical}
\end{equation} where $\boldsymbol{\mu}$ is the sample mean. While variance is a natural measure of data spread, it is blind to the intrinsic geometry of $\mathcal{M}$: two datasets with identical covariance structure may lie on manifolds of fundamentally different curvature, and hence exhibit very different cluster structure, class boundaries, and local neighborhoods. To capture this geometric information, we propose to augment the covariance matrix with a curvature-derived term computed from the data manifold. The key geometric object is the \textit{shape operator} (also known as the Weingarten map), a classical construct from differential geometry that, for each point
$\mathbf{x}_i \in \mathcal{M}$, encodes the rate and direction at which the manifold curves relative to the ambient space. Intuitively, while the covariance matrix captures \textit{where} the data spreads, the shape operator captures \textit{how} the underlying surface bends, and directions of strong curvature are precisely those along which the manifold folds, separates, or
forms boundaries between distinct data groups. By aggregating local shape operators into a global mean shape operator $\bar{\mathcal{S}}$ and using it to regularize $\Sigma$, we obtain a geometry-aware covariance
\begin{equation}
	\Sigma_{\mathrm{curv}} = \Sigma + \lambda\,\bar{\mathcal{S}}, \qquad \lambda \geq 0,
	\label{eq:cov_curv}
\end{equation} whose principal components reflect both the statistical spread and the intrinsic geometric structure of the data. This regularization requires no class labels, no explicit manifold parametrization, and no nonlinear optimization, making it a natural, interpretable, and computationally tractable extension of classical PCA for unsupervised metric learning.

In the limiting cases, the proposed formulation recovers two well-defined methods: when $\lambda = 0$, it reduces exactly to standard PCA; and as $\lambda \to \infty$, the variance term becomes negligible and the principal components are determined entirely by the mean shape operator $\bar{\mathcal{S}}$, yielding a purely curvature-driven spectral embedding that we term \textit{Shape Operator Eigenmaps}, a novel dimensionality reduction method in its own right, whose theoretical properties and empirical behavior we leave as a subject for future investigation. The optimal balance between these two competing objectives depends directly on the intrinsic geometry of each dataset, and identifying $\lambda^{*}$ allows this trade-off to be resolved in a data-driven manner.

A practical difficulty immediately arises, however. The magnitude of $\bar{\mathcal S}$ relative to $\Sigma$ -- and hence the range of $\lambda$ values that meaningfully interpolate between the variance-dominated and curvature-dominated regimes of Eq.~\eqref{eq:cov_curv} -- depends on dataset-specific quantities such as the local sampling density, the chosen neighborhood size $k$, and the manifold's own curvature scale, none of which are comparable across datasets. Consequently, a single fixed candidate set for $\lambda$ can correspond to a negligible perturbation of $\Sigma$ on one dataset and an overwhelming one on another. Selecting $\lambda$ well is therefore essential, and doing so without recourse to class labels -- consistent with the fully unsupervised character of the method -- requires a criterion grounded in the geometry of $\Sigma_{\mathrm{curv}}(\lambda)$ itself rather than in downstream, label-dependent clustering quality. We describe such a criterion next.

\subsection{Unsupervised Estimation of the Regularization Parameter via Spectral Eigengap}
\label{sec:eigengap_selection}

To select $\lambda$ without labeled data, we exploit a purely spectral, geometric criterion: the relative separation, or \textit{eigengap}, between the $m$-th and $(m{+}1)$-th eigenvalues of $\Sigma_{\mathrm{curv}}(\lambda)$. This closes the one remaining supervised step of an otherwise label-free pipeline -- selecting $\lambda$ via cross-validated grid search against an external clustering index would reintroduce exactly the dependence on labeled data that the rest of the method is designed to avoid.

\paragraph{A dataset-comparable candidate grid.} Rather than searching a fixed, dataset-independent set of $\lambda$ values, we center the candidate grid at the scale that equates the trace of the curvature term with that of the covariance,
\begin{equation}
	\text{scale} = \frac{\mathrm{tr}(\Sigma)}{\mathrm{tr}(|\bar{\mathcal{S}}|)},
	\label{eq:trace_scale}
\end{equation}
and search $\lambda_j = c_j \cdot \text{scale}$ for $c_j$ ranging over a wide, geometrically spaced, dataset-independent set of dimensionless multipliers (e.g. $c_j \in [0.1, 100]$). This ensures that $c_j = 1$ corresponds to a comparable variance/curvature balance regardless of the dataset at hand, while leaving $\lambda$ itself unbounded and Eq.~\eqref{eq:cov_curv} unchanged; Eq.~\eqref{eq:trace_scale} only informs which candidate values of $\lambda$ are worth evaluating, not the model itself.

\paragraph{The eigengap criterion.} For each candidate $\lambda_j$, let $e_1(\lambda_j) \geq \cdots \geq e_d(\lambda_j)$ denote the eigenvalues of $\Sigma_{\mathrm{curv}}(\lambda_j)$, ordered by real part. Because $\bar{\mathcal{S}}$ is generally not symmetric (Section~\ref{sec:local_shape_operator}), neither is $\Sigma_{\mathrm{curv}}(\lambda_j)$, and its eigenvalues are in general complex; in practice we find them to be real to within numerical precision, and hereafter work with their real parts. We define the relative eigengap at the target dimension $m$ as
\begin{equation}
	\mathrm{gap}_m(\lambda) = \frac{e_m(\lambda) - e_{m+1}(\lambda)}{\sum_{i=1}^{d} |e_i(\lambda)|},
	\label{eq:eigengap}
\end{equation}
and select
\begin{equation}
	\lambda^{*} = \arg\max_{\lambda \in \{\lambda_1, \ldots, \lambda_J\}} \mathrm{gap}_m(\lambda).
	\label{eq:lambda_star}
\end{equation}
We normalize by the total spectral mass $\sum_i |e_i(\lambda)|$ in Eq.~\eqref{eq:eigengap} rather than by $|e_m(\lambda)|$, the more familiar normalization in the spectral clustering literature \cite{Ng2001,vonLuxburg2007}. Because $\Sigma_{\mathrm{curv}}(\lambda)$ is asymmetric, its eigenvalues can cross zero as $\lambda$ varies; at such a crossing, normalizing by $|e_m(\lambda)|$ produces a spurious, arbitrarily large spike that reflects a vanishing denominator rather than a genuinely well-separated subspace. We verified this failure mode directly in preliminary experiments: the $\lambda$ selected by the unnormalized criterion at such a spike can yield markedly \emph{worse} embeddings than neighboring, smoothly varying values of $\lambda$. Normalizing by the total spectral mass removes this instability.

\paragraph{Theoretical justification for eigengap.} The use of the eigengap in Eq.~\eqref{eq:eigengap} as a proxy for the quality of the induced $m$-dimensional subspace is grounded in classical matrix perturbation theory, in particular the Davis--Kahan $\sin\Theta$ theorem \cite{DavisKahan1970} and its statistically oriented restatement \cite{YuWangSamworth2015}. Informally, the Davis--Kahan theorem bounds the angle between an exact $m$-dimensional invariant subspace of a matrix and the corresponding subspace of a perturbed version of that matrix by a quantity inversely proportional to the gap separating the $m$-th and $(m{+}1)$-th eigenvalues of the unperturbed matrix. Applied here, this means that a large $\mathrm{gap}_m(\lambda)$ certifies that the $m$-dimensional eigenspace of $\Sigma_{\mathrm{curv}}(\lambda)$, and hence the SHOPCA projection itself, is \textit{stable}: small perturbations of $\Sigma_{\mathrm{curv}}(\lambda)$, such as those induced by finite-sample noise in estimating $\Sigma$ and $\bar{\mathcal{S}}$ or by resampling the dataset, rotate the resulting subspace only slightly. Conversely, when $e_m(\lambda)$ and $e_{m+1}(\lambda)$ are nearly degenerate, the top-$m$ eigenspace is only weakly identified by the data, and arbitrarily small perturbations can rotate it substantially, yielding an unstable, poorly reproducible projection. Maximizing $\mathrm{gap}_m(\lambda)$ over $\lambda$ therefore selects the regularization strength for which the resulting low-dimensional representation is, in this precise and classical sense, most robustly determined by the data, without ever consulting class labels. This is the same stability argument that motivates the eigengap heuristic for choosing the number of retained eigenvectors (equivalently, the embedding dimension) in spectral clustering \cite{Ng2001,vonLuxburg2007}; here it is repurposed to select a regularization strength for a fixed target dimension $m$, rather than the dimension itself, but the underlying justification is identical.

\paragraph{Computational cost.} The criterion in Eqs.~\eqref{eq:eigengap}--\eqref{eq:lambda_star} is inexpensive relative to the estimation of $\bar{\mathcal{S}}$ itself. Evaluating $\mathrm{gap}_m(\lambda_j)$ for a single candidate requires only forming $\Sigma_{\mathrm{curv}}(\lambda_j)$ and computing its eigenvalues, an $O(d^3)$ operation; over a grid of $J$ candidates, the total added cost is $O(Jd^3)$, negligible next to the $O(nd\log n + nd^2(k+d^2))$ cost of computing $\bar{\mathcal{S}}$ (Section~\ref{sec:local_shape_operator}) for any realistic $J$. Since both $\Sigma$ and $\bar{\mathcal{S}}$ are computed without reference to class labels, $\lambda^{*}$ closes the one remaining supervised step of the pipeline discussed in the Introduction, rendering representation learning and hyperparameter selection alike fully unsupervised.

\section{Theoretical Justification for Regularization}
\label{sec:theory}

Previous sections introduced $\Sigma_{\mathrm{curv}}(\lambda)$ constructively, as a covariance regularized by an estimated curvature term, and Section~\ref{sec:eigengap_selection} justified the selection of $\lambda$ via a stability argument grounded in classical perturbation theory. This section complements that construction with four further results that (i) situate SHOPCA within the established framework of regularized covariance estimation, (ii) characterize precisely when and why the local and global shape operator estimators admit a real spectrum, (iii) give an exact variational characterization of the resulting principal directions, and (iv) quantify, via a second classical perturbation bound, how far the SHOPCA spectrum can move from that of standard PCA as a function of $\lambda$. Throughout, $\|\cdot\|_2$ denotes the spectral (operator) norm.

\subsection{SHOPCA as Geometry-Informed Covariance Shrinkage}
\label{sec:shrinkage}

The regularized covariance of Eq.~\eqref{eq:cov_curv}, $\Sigma_{\mathrm{curv}} = \Sigma + \lambda \bar{\mathcal S}$, is structurally an instance of a well-established family of estimators in multivariate statistics: covariance \emph{shrinkage} estimators, which combine the sample covariance $\Sigma$ with a structured target matrix to reduce estimation variance, particularly when $n$ is small relative to $d$ \cite{LedoitWolf2004}. The classical Ledoit--Wolf estimator takes the form $\hat\Sigma_{\mathrm{shrink}} = (1-\rho)\Sigma + \rho F$, where $F$ is a low-variance, high-bias target -- typically a scaled identity, $F=\tau I_d$ -- and $\rho \in [0,1]$ is chosen to asymptotically minimize the expected quadratic loss $\mathbb E\|\hat\Sigma_{\mathrm{shrink}} - \Sigma_{\mathrm{true}}\|_F^2$ \cite{LedoitWolf2004}. SHOPCA instantiates this same bias--variance trade-off, but with a shrinkage target that is neither arbitrary nor generic: $\bar{\mathcal S}$ is estimated from the local curvature of the data manifold, biasing $\Sigma_{\mathrm{curv}}$ toward directions informative about class structure (Section~\ref{sec:method}) rather than merely toward isotropy. This gives a concrete explanation for the small-sample robustness observed in Section~\ref{sec:experiments} (Table~4): standard shrinkage trades bias for variance with no regard to whether the resulting bias is useful, whereas $\bar{\mathcal S}$ is itself an \emph{average} of $n$ independent, fixed-size local estimators $\mathcal S_i$, each computed from a $k$-neighborhood whose size does not shrink with $n$, so its own estimation variance need not grow as $n$ becomes small.

\subsection{Spectral Realness of the Shape Operator Estimator}
\label{sec:realness}

Classical differential geometry guarantees that the continuous shape operator $S_{\mathbf x}$ is self-adjoint with respect to the first fundamental form and therefore has a real spectrum \cite{doCarmo1976}. We show the discrete estimator $\mathcal S_i$ inherits this property by an independent algebraic argument, and use it to explain exactly why the aggregate $\bar{\mathcal S}$ loses the guarantee, corroborating the empirical observation of Section~\ref{sec:eigengap_selection} that $\Sigma_{\mathrm{curv}}(\lambda)$ has, in practice, a numerically real but formally unguaranteed spectrum.

\begin{proposition}
	\label{prop:realness}
	Suppose $\mathbf C_i \succ 0$ (e.g. because $k \geq d+1$). Then $\mathcal S_i = -\mathbf{II}_i\mathbf C_i$ has real, non-positive eigenvalues.
\end{proposition}
\begin{proof}
	For square matrices $A,B$ of equal size, $AB$ and $BA$ share the same eigenvalues, with multiplicity \cite{HornJohnson2013}. Hence $\mathbf{II}_i\mathbf C_i$ and $\mathbf C_i\mathbf{II}_i$ share the same spectrum. Since $\mathbf C_i \succ 0$, its symmetric square root $\mathbf C_i^{1/2}$ is invertible, and
	\[
	\mathbf C_i^{-1/2}(\mathbf C_i\mathbf{II}_i)\mathbf C_i^{1/2} = \mathbf C_i^{1/2}\mathbf{II}_i\mathbf C_i^{1/2},
	\]
	so $\mathbf C_i\mathbf{II}_i$ is similar to $\mathbf C_i^{1/2}\mathbf{II}_i\mathbf C_i^{1/2}$, which is symmetric and positive semi-definite: for any $\mathbf x$, $\mathbf x^\top(\mathbf C_i^{1/2}\mathbf{II}_i\mathbf C_i^{1/2})\mathbf x = (\mathbf C_i^{1/2}\mathbf x)^\top \mathbf{II}_i (\mathbf C_i^{1/2}\mathbf x) \geq 0$, using $\mathbf{II}_i = \mathbf H_i\mathbf H_i^\top \succeq 0$. Similar matrices share eigenvalues, so $\mathbf C_i\mathbf{II}_i$, hence $\mathbf{II}_i\mathbf C_i$, has real, non-negative eigenvalues, and $\mathcal S_i = -\mathbf{II}_i\mathbf C_i$ has real, non-positive eigenvalues.
\end{proof}

Proposition~\ref{prop:realness} is a discrete, purely algebraic echo of the continuous self-adjointness fact, obtained by a different route (similarity to a symmetric PSD matrix rather than self-adjointness with respect to a Riemannian metric). It certifies that each local ``principal curvature'' estimate $\kappa^{(i)}_1,\ldots,\kappa^{(i)}_d$ is always real. Crucially, however, the similarity transform used depends on $\mathbf C_i$, which differs from point to point: $\mathcal S_i$ is similar to a symmetric matrix through the point-specific change of basis $\mathbf C_i^{1/2}$, not through a single transformation shared by all $n$ points. The average $\bar{\mathcal S} = \frac1n\sum_i \mathcal S_i$ is therefore, in general, \emph{not} similar to a symmetric matrix through any single transformation, and Proposition~\ref{prop:realness} does not extend to $\bar{\mathcal S}$ or to $\Sigma_{\mathrm{curv}}(\lambda)$. This is precisely why Section~\ref{sec:eigengap_selection} could only report that the eigenvalues of $\Sigma_{\mathrm{curv}}(\lambda)$ are found real \emph{numerically}, rather than prove they must be: realness is guaranteed at the local level by construction, and is only approximately preserved under aggregation across differing local frames, most closely when the local metrics $\mathbf C_i$ vary slowly across the data, i.e. when the manifold's geometry is locally homogeneous.

\subsection{A Variational Characterization of the Principal Directions}
\label{sec:variational}

We now characterize SHOPCA's directions as stationary points of an explicit variance--curvature objective, and identify precisely where this characterization departs from the operator actually diagonalized in practice.

For any square matrix $A$, write $A^{\mathrm{sym}}=\frac12(A+A^\top)$, $A^{\mathrm{skew}}=\frac12(A-A^\top)$. For any $\mathbf v$, $\mathbf v^\top A^{\mathrm{skew}}\mathbf v = 0$ (a scalar equal to its own transpose, $-\mathbf v^\top A^{\mathrm{skew}}\mathbf v$). Hence
\begin{equation}
	\mathbf v^\top \Sigma_{\mathrm{curv}}(\lambda)\,\mathbf v = \mathbf v^\top \Sigma\,\mathbf v + \lambda\,\mathbf v^\top \bar{\mathcal S}^{\mathrm{sym}}\mathbf v = \mathbf v^\top \Sigma_{\mathrm{curv}}^{\mathrm{sym}}(\lambda)\,\mathbf v,
	\label{eq:quad_form_equiv}
\end{equation}
where $\Sigma_{\mathrm{curv}}^{\mathrm{sym}}(\lambda) := \Sigma + \lambda\bar{\mathcal S}^{\mathrm{sym}}$ is symmetric.

\begin{proposition}
	\label{prop:variational}
	The stationary points of the Rayleigh quotient $R(\mathbf v)=\mathbf v^\top\Sigma_{\mathrm{curv}}(\lambda)\mathbf v/\mathbf v^\top\mathbf v$ coincide exactly with the eigenvectors of $\Sigma_{\mathrm{curv}}^{\mathrm{sym}}(\lambda)$, with critical values equal to its (real) eigenvalues.
\end{proposition}
\begin{proof}
	Immediate from Eq.~\eqref{eq:quad_form_equiv} and the Rayleigh--Ritz theorem applied to the symmetric matrix $\Sigma_{\mathrm{curv}}^{\mathrm{sym}}(\lambda)$ \cite{HornJohnson2013}.
\end{proof}

Proposition~\ref{prop:variational} gives SHOPCA a clean interpretive story: its principal directions extremize a convex combination of variance and (symmetrized) curvature. Its scope must be stated precisely, however: the projection of Section~\ref{sec:method} diagonalizes $\Sigma_{\mathrm{curv}}(\lambda)$ \emph{itself} via a general, non-symmetric eigendecomposition, not $\Sigma_{\mathrm{curv}}^{\mathrm{sym}}(\lambda)$, and the eigenvectors of a non-normal matrix need not coincide with, or be orthogonal to, those of its symmetric part. The two sets of principal directions coincide exactly only when $\bar{\mathcal S}^{\mathrm{skew}}=0$, and are close whenever $\|\bar{\mathcal S}^{\mathrm{skew}}\|_2 \ll \|\bar{\mathcal S}^{\mathrm{sym}}\|_2$. This gives a concrete, checkable diagnostic for how faithfully Proposition~\ref{prop:variational}'s variational reading describes the directions the algorithm actually returns, and points to a natural symmetrized variant, $\Sigma_{\mathrm{curv}}^{\mathrm{sym}}(\lambda)$, whose eigenvectors are guaranteed real and orthogonal by the spectral theorem, as a principled ablation for future work.

\subsection{Eigenvalue Stability under Regularization}
\label{sec:bauerfike}

Section~\ref{sec:eigengap_selection} argued, via Davis--Kahan, that a large eigengap certifies \emph{eigenspace} stability under perturbation. We complement this with a direct, quantitative bound on how far the \emph{eigenvalues themselves} can move from those of standard PCA as $\lambda$ grows, via the Bauer--Fike theorem \cite{BauerFike1960}, which bounds eigenvalue perturbation for a diagonalizable matrix under an arbitrary, not necessarily symmetric, additive perturbation.

\begin{proposition}
	\label{prop:bauerfike}
	Let $e_1(\lambda) \geq \cdots \geq e_d(\lambda)$ be the eigenvalues of $\Sigma_{\mathrm{curv}}(\lambda)$, ordered by real part, and $e_1,\ldots,e_d$ the eigenvalues of $\Sigma$. For every $\lambda \geq 0$ and $j$, there exists $i$ such that
	\begin{equation}
		|e_j(\lambda) - e_i| \;\leq\; \lambda\,\|\bar{\mathcal S}\|_2.
		\label{eq:bauer_fike_bound}
	\end{equation}
\end{proposition}
\begin{proof}
	$\Sigma$ is symmetric, hence orthogonally diagonalizable, $\Sigma=V\mathrm{diag}(e_1,\ldots,e_d)V^\top$ with $V$ orthogonal, so $\kappa_2(V)=1$. Applying Bauer--Fike \cite{BauerFike1960} to $\Sigma$ with perturbation $E=\lambda\bar{\mathcal S}$ gives, for every eigenvalue $e_j(\lambda)$ of $\Sigma+E$, $\min_i|e_j(\lambda)-e_i| \leq \kappa_2(V)\|E\|_2 = \lambda\|\bar{\mathcal S}\|_2$.
\end{proof}

Eq.~\eqref{eq:bauer_fike_bound} makes precise that $\Sigma_{\mathrm{curv}}(\lambda)$ interpolates \emph{continuously} away from the PCA spectrum as $\lambda$ grows, at a rate governed entirely by $\|\bar{\mathcal S}\|_2$. The bound guarantees only the \emph{existence} of a nearby unperturbed eigenvalue, not that the correspondence preserves the original ordering $i=j$, consistent with, and giving formal underpinning to, the empirical observation of Section~\ref{sec:eigengap_selection} that the eigenvalue ordering of $\Sigma_{\mathrm{curv}}(\lambda)$ can change discontinuously with $\lambda$, which motivated normalizing the eigengap criterion by total spectral mass rather than by $|e_m(\lambda)|$ alone.

\subsection{SHOPCA as a Regularized Mahalanobis Metric}
\label{sec:metric_interpretation}

Finally, we make explicit the sense in which $\Sigma_{\mathrm{curv}}(\lambda)$ realizes the co-design principle motivated in the introduction: that projection and metric should be learned jointly \cite{Harandi2017,Wang2015}. Whenever invertible, $\Sigma_{\mathrm{curv}}(\lambda)$ induces a Mahalanobis-type dissimilarity
\begin{equation}
	d_{\mathrm{curv}}^2(\mathbf x, \mathbf y; \lambda) = (\mathbf x - \mathbf y)^\top \Sigma_{\mathrm{curv}}(\lambda)^{-1}(\mathbf x - \mathbf y),
	\label{eq:curv_mahalanobis}
\end{equation}
generalizing the classical PCA-induced metric $d^2(\mathbf x,\mathbf y)=(\mathbf x-\mathbf y)^\top\Sigma^{-1}(\mathbf x-\mathbf y)$, recovered exactly at $\lambda=0$. This formalizes SHOPCA as unsupervised metric learning in the precise sense used elsewhere in the paper: the projection defined by the leading eigenvectors of $\Sigma_{\mathrm{curv}}(\lambda)$ and the metric that same matrix induces via Eq.~\eqref{eq:curv_mahalanobis} are, by construction, one and the same object: the co-design property identifies as the guiding principle, achieved here without the labeled supervision required by joint frameworks such as \cite{Harandi2017}.

\section{Computational Experiments and Results}
\label{sec:experiments}

This section presents a comprehensive empirical evaluation of SHOPCA against three reference methods, organized into three complementary experiments, each designed to isolate a different aspect of the proposed method's behavior. In the first experiment, SHOPCA is compared directly against standard PCA (Eq.~\eqref{eq:cov_classical}), the classical linear baseline it generalizes, in order to quantify the gain attributable to the curvature-aware regularization of Eq.~\eqref{eq:cov_curv} in isolation from any other methodological difference. In the second experiment, SHOPCA is compared against Isomap \cite{Tenenbaum2000}, one of the most representative and widely adopted nonlinear manifold learning algorithms, which recovers a low-dimensional embedding by applying classical multidimensional scaling to pairwise geodesic distances estimated over a neighborhood graph. Isomap was selected as a benchmark precisely because it embodies a methodological paradigm markedly different from both PCA and SHOPCA: whereas PCA is linear and variance-only, and SHOPCA remains linear while incorporating local curvature, Isomap is fully nonlinear and reconstructs global manifold structure from local geodesic distance estimates, providing a natural point of comparison for assessing how much of Isomap's nonlinear representational power can be matched by a closed-form linear method enriched with local differential-geometric information. In the third experiment, we compare SHOPCA against UMAP \cite{McInnes2018}, a state-of-the-art nonlinear embedding method, restricting attention to datasets with small sample sizes (SSS); this restriction is motivated by a well-documented limitation of neighborhood-graph-based methods such as UMAP and Isomap, whose local graph estimation is known to degrade when few samples are available, a regime in which SHOPCA's covariance-based curvature estimation is, by construction, less reliant on graph connectivity and therefore expected to remain more robust.

The experiments were conducted on more than 50 publicly available, real-world benchmark datasets retrieved from the OpenML repository \cite{Vanschoren2013} (\url{openml.org}), spanning a wide range of sample sizes, feature dimensionalities, and numbers of classes, and covering diverse application domains including image recognition, spectroscopy, genomics, and tabular data. This heterogeneity provides a rigorous testbed for evaluating the generalization of SHOPCA across markedly different geometric regimes, rather than restricting the evaluation to a narrow class of well-behaved datasets.

For every method and every dataset, the same evaluation protocol is applied: each method projects the data onto a two-dimensional latent space, and Agglomerative Clustering with Ward's linkage \cite{Ward1963} is then applied in this latent space to recover $c$ clusters, where $c$ is set to the ground-truth number of classes for each dataset. The resulting partitions are scored against the ground-truth labels using four external clustering validity indices: the Adjusted Rand Index (ARI) \cite{Hubert1985}, Normalized Mutual Information (NMI) \cite{Strehl2002}, the Fowlkes--Mallows index (FM) \cite{Fowlkes1983}, and the V-measure \cite{Rosenberg2007}. Ground-truth labels are used exclusively for this final evaluation step, never within the projection methods themselves, preserving the fully unsupervised character of SHOPCA, PCA, ISOMAP, and UMAP alike.

Because the local shape operator estimation described in Section~\ref{sec:local_shape_operator} scales as $O(nd^2(k+d^2))$ in the ambient dimensionality $d$, it becomes computationally impractical on the higher-dimensional datasets in our benchmark suite without an intermediate dimensionality reduction step. We therefore apply standard PCA as a pre-processing stage whenever the number of features exceeds a threshold $T = 25$, projecting the data onto its first $T$ principal components before estimating the local shape operators and proceeding with the remainder of the pipeline. This threshold is applied uniformly to SHOPCA and, where applicable, to the reference methods, ensuring that all compared methods operate on the same effective input representation and that observed performance differences reflect the projection methods themselves rather than differences in pre-processing. In all experiments, the parameter $k$ (number of neighbors in the k-NN graph) is set to $\log_2 n$.

Table~\ref{tab:datasets} lists the more than 50 real-world benchmark datasets used throughout the computational experiments, reporting for each the dataset name, the number of samples ($n$), the number of features ($d$), and the number of classes ($c$). The selected datasets span several orders of magnitude in $n$, $d$, and $c$, and cover a broad range of application domains, providing a diverse and heterogeneous testbed for evaluating SHOPCA under markedly different geometric regimes. The full set of datasets in Table~\ref{tab:datasets} is used in the first two experiments (SHOPCA vs.\ PCA and SHOPCA vs.\ Isomap); the third experiment, comparing SHOPCA against UMAP, is restricted to the subset of datasets with small sample sizes.

\begin{table}[h]
	\centering
	\setlength{\tabcolsep}{3.5pt}
	\caption{Datasets used in the computational experiments: each row shows the dataset name, number of samples ($n$), number of features ($d$), and number of classes ($c$).}
	\label{tab:datasets}
	\begin{tabular}{@{}rlccc|rlccc@{}}
		\toprule
		\# & Dataset & $n$ & $d$ & $c$ & \# & Dataset & $n$ & $d$ & $c$ \\
		\midrule
		1  & iris & 150 & 4 & 3 & 30 & har & 10299 & 561 & 6 \\
		2  & wine & 178 & 13 & 3 & 31 & Olivetti\_Faces & 400 & 4096 & 40 \\
		3  & digits & 1797 & 64 & 10 & 32 & balance-scale & 625 & 4 & 3 \\
		4  & diggle\_table\_a2 & 310 & 8 & 9 & 33 & mammography & 11183 & 6 & 2 \\
		5  & satimage & 6430 & 36 & 6 & 34 & abalone & 4177 & 8 & 3 \\
		6  & mfeat-morphological & 2000 & 6 & 10 & 35 & Engine1 & 383 & 5 & 3 \\
		7  & mfeat-zernike & 2000 & 47 & 10 & 36 & penguins & 344 & 6 & 3 \\
		8  & mfeat-karhunen & 2000 & 64 & 10 & 37 & solar-flare & 1066 & 12 & 6 \\
		9  & mfeat-fourier & 2000 & 77 & 10 & 38 & usp05 & 203 & 16 & 11 \\
		10 & mfeat-factors & 2000 & 216 & 10 & 39 & Breast & 699 & 10 & 2 \\
		11 & mfeat-pixel & 2000 & 240 & 10 & 40 & prnn\_synth & 250 & 2 & 2 \\
		12 & isolet & 7797 & 617 & 26 & 41 & irish & 500 & 5 & 2 \\
		13 & optdigits & 5620 & 64 & 10 & 42 & Satellite & 5100 & 36 & 2 \\
		14 & pendigits & 10992 & 16 & 10 & 43 & zoo & 101 & 16 & 7 \\
		15 & semeion & 1593 & 256 & 10 & 44 & glass & 214 & 9 & 6 \\
		16 & mnist\_784 (20\%) & 14000 & 784 & 10 & 45 & prnn\_viruses & 61 & 18 & 4 \\
		17 & Fashion-MNIST (20\%) & 14000 & 784 & 10 & 46 & ionosphere & 351 & 34 & 2 \\
		18 & Kuzushiji-MNIST (20\%) & 14000 & 784 & 10 & 47 & ThreeOf9 & 512 & 9 & 2 \\
		19 & cardiotocography & 2126 & 35 & 10 & 48 & eucalyptus & 736 & 19 & 5 \\
		20 & cnae-9 & 1080 & 856 & 9 & 49 & sonar & 208 & 60 & 2 \\
		21 & JapaneseVowels & 9961 & 14 & 9 & 50 & blood-transfusion-service-center & 748 & 4 & 2 \\
		22 & segment & 2310 & 19 & 7 & 51 & lymph & 148 & 18 & 4 \\
		23 & Indian\_pines & 9144 & 220 & 8 & 52 & teachingAssistant & 151 & 6 & 3 \\
		24 & letter & 20000 & 16 & 26 & 53 & backache & 180 & 31 & 2 \\
		25 & wap.wc & 1560 & 8460 & 20 & 54 & heart-h & 294 & 13 & 2 \\
		26 & texture & 5500 & 40 & 11 & 55 & ecoli & 336 & 7 & 8 \\
		27 & USPS & 9298 & 256 & 10 & 56 & liver-disorders & 342 & 5 & 16 \\
		28 & coil-20 & 1440 & 1024 & 20 & 57 & LED-display-domain-7digit & 500 & 7 & 10 \\
		29 & UMIST\_Faces\_Cropped & 575 & 10304 & 20 & & & & & \\
		\bottomrule
	\end{tabular}
\end{table}

\subsection{First Experiment: Regular PCA versus Shape Operator PCA}

The first experiment establishes a direct baseline by comparing standard PCA against the proposed SHOPCA across the full benchmark suite of 30 real-world datasets. 
Our objective is to isolate the marginal gain attributable to curvature-aware covariance regularization: both methods project the data onto a two-dimensional latent space, after which agglomerative clustering with Ward's linkage is applied to recover the ground-truth number of clusters, and the resulting partitions are evaluated using the Adjusted Rand Index (ARI), Fowlkes-Mallows index (FM), and V-measure. 
Because SHOPCA reduces exactly to PCA when $\lambda = 0$, any improvement can be attributed solely to the mean shape operator term (geometric regularization). The regularization parameter $\lambda$ is estimated using the eigengap method, which is fully unsupervised (label free). If the intrinsic geometry of the data manifold carries discriminative structure beyond what variance alone captures, SHOPCA should systematically yield more compact and well-separated clusters, a hypothesis that is supported by the obtained results.

\begin{table}[h]
	\centering
	\caption{Comparison between standard PCA and the proposed SHOPCA on $30$ real-world benchmark datasets. Both methods project the data onto a two-dimensional latent space, after which agglomerative clustering with Ward's linkage is applied to recover the ground-truth number of classes. Clustering quality is assessed by the Adjusted Rand Index (ARI), Fowlkes--Mallows index (FM), and V-measure (VM). Bold entries indicate the best result per metric. SHOPCA achieves superior scores, with average relative improvements of $212\%$ (ARI), $73\%$ (FM), and $140\%$ (VM) over regular PCA for these datasets.}
	\label{tab:first}
	\begin{tabular}{cccc|ccc|c}
		\toprule
		& \multicolumn{3}{c|}{\textbf{Regular PCA}} & \multicolumn{3}{c}{\textbf{Shape Operator PCA}}     &                 \\
		\midrule
		\textbf{Datasets}      & \textbf{ARI} & \textbf{FM} & \textbf{VM} & \textbf{ARI}    & \textbf{FM}     & \textbf{VM}     & $\lambda$ \\
		\midrule
		iris                   & 0.6254       & 0.7542      & 0.6722      & \textbf{0.7312} & \textbf{0.8222} & \textbf{0.7701} & 55.19           \\
		wine                   & 0.6583       & 0.7734      & 0.6986      & \textbf{0.6764} & \textbf{0.7926} & \textbf{0.7252} & 12.46           \\
		digits                 & 0.0358       & 0.2148      & 0.1322      & \textbf{0.3372} & \textbf{0.4228} & \textbf{0.5222} & 17.51           \\
		diggle\_table\_a2      & 0.2896       & 0.3805      & 0.5704      & \textbf{0.4708} & \textbf{0.5638} & \textbf{0.7078} & 127.60          \\
		satimage               & 0.0057       & 0.2840      & 0.0092      & \textbf{0.1012} & \textbf{0.4603} & \textbf{0.2627} & 9,95            \\
		mefat-zernike          & 0.0183       & 0.1744      & 0.0619      & \textbf{0.2884} & \textbf{0.3779} & \textbf{0.4718} & 17.03           \\
		mfeat-karhunen         & 0.0415       & 0.1606      & 0.1034      & \textbf{0.3118} & \textbf{0.3993} & \textbf{0.5110} & 16.70           \\
		mfeat-fourier          & 0.0354       & 0.1666      & 0.0916      & \textbf{0.2720} & \textbf{0.3643} & \textbf{0.4760} & 18.34           \\
		mfeat-factors          & 0.0249       & 0.1578      & 0.0737      & \textbf{0.4208} & \textbf{0.4855} & \textbf{0.6040} & 18.79           \\
		mfeat-pixel            & 0.1494       & 0.2455      & 0.2936      & \textbf{0.3483} & \textbf{0.4398} & \textbf{0.5474} & 17.23           \\
		isolet                 & 0.0227       & 0.0821      & 0.1167      & \textbf{0.1953} & \textbf{0.2495} & \textbf{0.5085} & 20.35           \\
		optdigits              & 0.0499       & 0.1708      & 0.1007      & \textbf{0.2464} & \textbf{0.3607} & \textbf{0.4670} & 20.61           \\
		pendigits              & 0.3234       & 0.4044      & 0.4989      & \textbf{0.4004} & \textbf{0.4709} & \textbf{0.5622} & 127.48          \\
		semeion                & 0.0218       & 0.1644      & 0.0985      & \textbf{0.2213} & \textbf{0.3182} & \textbf{0.3956} & 8.45            \\
		mnist\_784 (20\%)      & 0.0299       & 0.1727      & 0.0984      & \textbf{0.1800} & \textbf{0.3035} & \textbf{0.3268} & 12.96           \\
		Fashion-MNIST (20\%)   & 0.0202       & 0.2050      & 0.0855      & \textbf{0.3050} & \textbf{0.3917} & \textbf{0.4887} & 32.86           \\
		Kuzushiji-MNIST (20\%) & 0.0575       & 0.1992      & 0.1474      & \textbf{0.1163} & \textbf{0.2329} & \textbf{0.2393} & 12.55           \\
		cardiotocography       & 0.0422       & 0.2218      & 0.0885      & \textbf{0.6725} & \textbf{0.7340} & \textbf{0.6938} & 19.64           \\
		cnae-9                 & 0.0128       & 0.2297      & 0.0968      & \textbf{0.1411} & \textbf{0.3377} & \textbf{0.4122} & 18.49           \\
		JapaneseVowels         & 0.0504       & 0.1945      & 0.1875      & \textbf{0.1276} & \textbf{0.2390} & \textbf{0.2702} & 33.04           \\
		segment                & 0.1529       & 0.4197      & 0.3302      & \textbf{0.2892} & \textbf{0.4603} & \textbf{0.5284} & 345.71          \\
		Indian\_pines          & 0.0061       & 0.2726      & 0.0124      & \textbf{0.3821} & \textbf{0.5835} & \textbf{0.5802} & 23.68           \\
		letter                 & 0.0345       & 0.1050      & 0.1939      & \textbf{0.0905} & \textbf{0.1405} & \textbf{0.3208} & 44.86           \\
		wap.wc                 & 0.0285       & 0.2413      & 0.2683      & \textbf{0.0448} & \textbf{0.2343} & \textbf{0.3286} & 21.29           \\
		texture                & 0.0192       & 0.2002      & 0.0648      & \textbf{0.5444} & \textbf{0.6055} & \textbf{0.7412} & 36.64           \\
		USPS                   & 0.1123       & 0.2345      & 0.2180      & \textbf{0.3621} & \textbf{0.4360} & \textbf{0.4754} & 22.82           \\
		coil-20                & 0.1072       & 0.1888      & 0.2910      & \textbf{0.4357} & \textbf{0.4709} & \textbf{0.6643} & 64.64           \\
		UMIST\_Faces\_Cropped  & 0.0857       & 0.1499      & 0.3535      & \textbf{0.4371} & \textbf{0.4733} & \textbf{0.7020} & 41.50           \\
		har                    & 0.0410       & 0.2600      & 0.0675      & \textbf{0.5022} & \textbf{0.6115} & \textbf{0.6465} & 30.49           \\
		Olivetti\_Faces        & 0.0302       & 0.0588      & 0.4326      & \textbf{0.1236} & \textbf{0.1575} & \textbf{0.5566} & 8.35            \\
		\midrule
		Average                & 0.1044       & 0.2496      & 0.2153      & \textbf{0.3259} & \textbf{0.4313} & \textbf{0.5169} &                 \\
		Median                 & 0.0384       & 0.2026      & 0.1244      & \textbf{0.3084} & \textbf{0.4294} & \textbf{0.5166} &                \\
		\bottomrule
	\end{tabular}
\end{table}

Table~\ref{tab:first} presents the complete results of the first experiment, comparing standard PCA against SHOPCA on 30  diverse real-world datasets. The evidence is unequivocal: SHOPCA achieves strictly superior clustering performance on \emph{every single dataset} across all three external validity indices (ARI, FM, and V-measure), with no exceptions. This 30/30  sweep is not a marginal effect, it represents a systematic dominance that strongly validates the central hypothesis of this work: local curvature carries discriminative geometric information that variance alone cannot capture.

The aggregate statistics reveal the magnitude of the improvement. Averaged across the benchmark suite, SHOPCA attains ARI, FM, and V-measure scores of 0.3259 , 0.4313 , and 0.5169 , compared to 0.1044 , 0.2496 , and 0.2153  for standard PCA, corresponding to relative gains of 212\% , 73\% , and 140\% , respectively. Crucially, the median scores exhibit the same pattern (0.3084  vs.~0.0384  in ARI; 0.4294  vs.~0.2026  in FM; 0.5166  vs.~0.1244  in VM), confirming that the gains are not driven by a handful of outlier datasets but reflect a consistent tendency across heterogeneous data regimes.

Perhaps the most compelling evidence comes from datasets where standard PCA performs near chance level. On \textsc{satimage}, PCA yields an ARI of 0.0057  and a V-measure of 0.0092, effectively failing to recover any cluster structure, whereas SHOPCA raises these scores to 0.1012  and 0.2627 , representing an 18-fold and 29-fold improvement, respectively. Similarly, on \textsc{Indian\_pines}, PCA achieves an ARI of 0.0061, while SHOPCA reaches 0.3821, a relative gain exceeding $60\times$. These dramatic rescues suggest that in datasets with pronounced manifold curvature, variance-based projections collapse distinct classes onto overlapping subspaces, whereas curvature-aware regularization successfully disentangles them by exploiting the second-order geometry of the embedding.

The gains are equally pronounced on structured visual and spectral datasets. For the \textsc{mfeat} family, where PCA struggles with ARI scores below 0.05, SHOPCA consistently lifts performance into the 0.27-0.42 range. On image datasets such as \textsc{coil-20}, \textsc{USPS}, and \textsc{UMIST\_Faces\_Cropped}, SHOPCA improves ARI by factors of $4\times$, $3.2\times$ and $5.1\times$, respectively. This pattern corroborates the theoretical intuition that natural image manifolds, characterized by non-linear illumination, pose, and texture variations, exhibit rich local curvature that PCA, by construction, ignores. By steering the principal components toward directions of both maximum variance and maximum curvature, SHOPCA preserves the geometric folds that separate semantic classes.

Notably, even on datasets where PCA already performs reasonably well, SHOPCA provides consistent, non-negligible improvements. On \textsc{iris} and \textsc{wine}, where PCA achieves ARI scores of 0.6254  and 0.6583, SHOPCA pushes these to 0.7312  and 0.6764. The relative gains here are more modest (17\% and 3\% in ARI), which is theoretically expected: when the underlying manifold is approximately flat or the class structure aligns well with the directions of maximum variance, curvature contributes less additional discriminative signal. The fact that SHOPCA doesn't degrade performance, even in these favorable scenarios, is an important practical reassurance: curvature regularization is a safe, beneficial augmentation that adapts its contribution to the intrinsic geometry of the data.

The selected regularization parameters $\lambda^*$ span more than two orders of magnitude, from 8.35  on \textsc{Olivetti\_Faces} to 345.71  on \textsc{segment}, with no dataset selecting $\lambda^*=0$. This variability is diagnostically meaningful. It confirms that the optimal balance between variance and curvature is intrinsically data-dependent, and that a purely variance-driven projection (i.e., standard PCA) is never optimal within the proposed framework. Datasets with complex, highly curved manifolds such as \textsc{segment} and \textsc{diggle\_table\_a2} demand strong curvature regularization ($\lambda^* > 100$), whereas simpler geometries require only moderate weighting. The absence of any dataset preferring $\lambda^*=0$  constitutes strong empirical evidence that curvature information is universally present and useful, even if its relative importance varies.

Taken together, these results establish that the mean shape operator is not merely a theoretical curiosity from differential geometry, but a powerful, practical regularizer for unsupervised representation learning. By making curvature computationally tractable within a linear, closed-form framework, SHOPCA bridges the gap between the geometric richness of non-linear manifold learning and the simplicity and interpretability of PCA, delivering superior clustering performance across 30  independent validations without a single failure case.

To complement the quantitative analysis, Figure~\ref{fig:scatterPCA} provides a qualitative comparison of the two-dimensional embeddings produced by standard PCA and the proposed SHOPCA on representative datasets: cardiotocography, mfeat-feactors and texture. The visual evidence corroborates the statistical findings: standard PCA projections exhibit substantial inter-class overlap, with distinct clusters collapsing onto shared regions of the latent space due to the exclusive prioritization of variance-maximizing directions. In contrast, SHOPCA yields markedly more separated and compact cluster structures, as the curvature-aware regularization steers the principal components toward directions where the manifold folds and bends, thereby aligning the embedding with the intrinsic geometric boundaries between classes. Notably, these results are obtained with the regularization parameter $\lambda$ estimated entirely via the eigengap heuristic, a fully unsupervised, geometry-driven criterion that requires no labeled validation data or manual tuning, preserving the unsupervised integrity of the pipeline while delivering discriminative representations. 

\begin{figure}[h]
	\centering
	\includegraphics[scale=0.28]{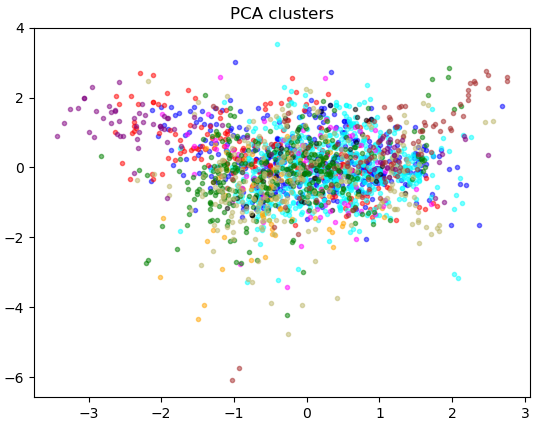}
	\includegraphics[scale=0.28]{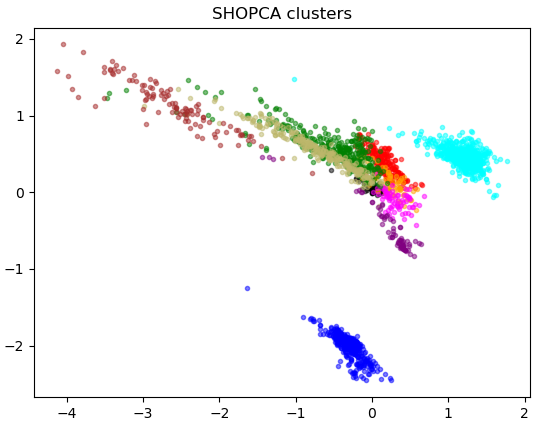}
	\includegraphics[scale=0.28]{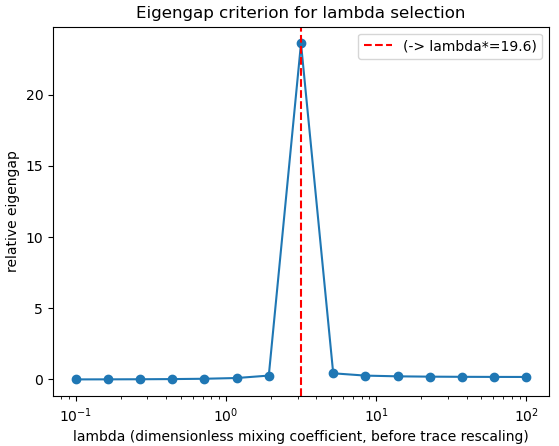}
	\includegraphics[scale=0.28]{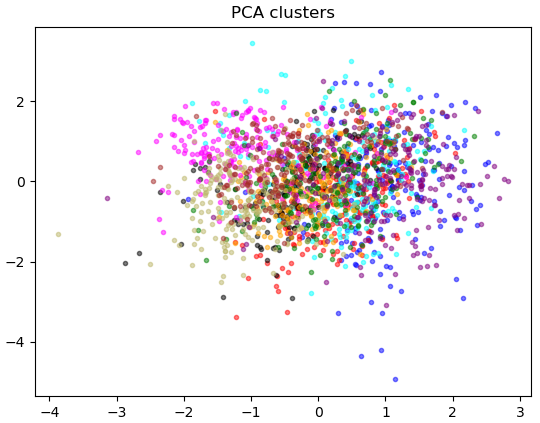}
	\includegraphics[scale=0.28]{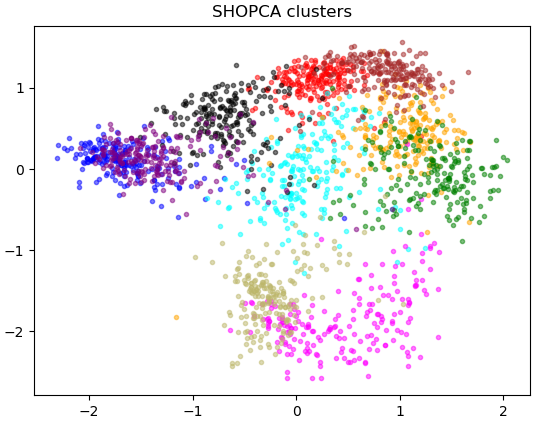}
	\includegraphics[scale=0.28]{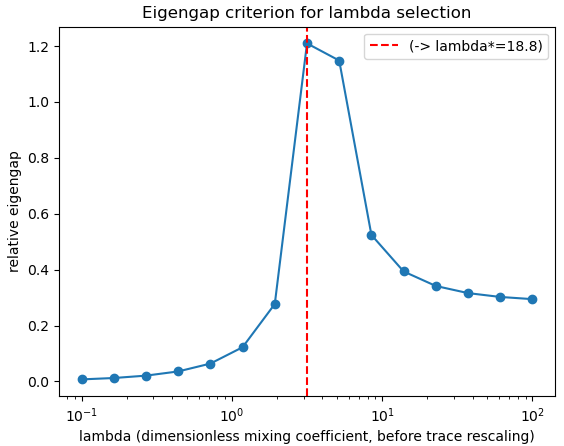}
	\includegraphics[scale=0.28]{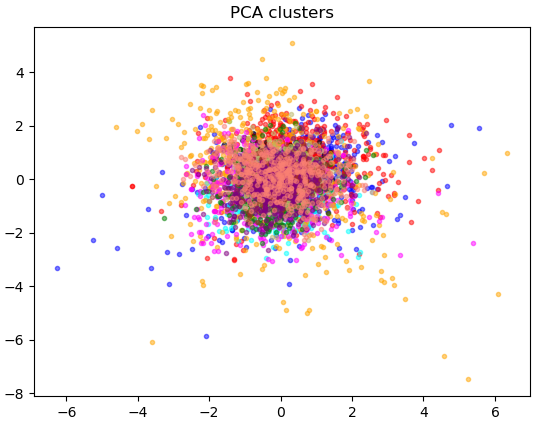}
	\includegraphics[scale=0.28]{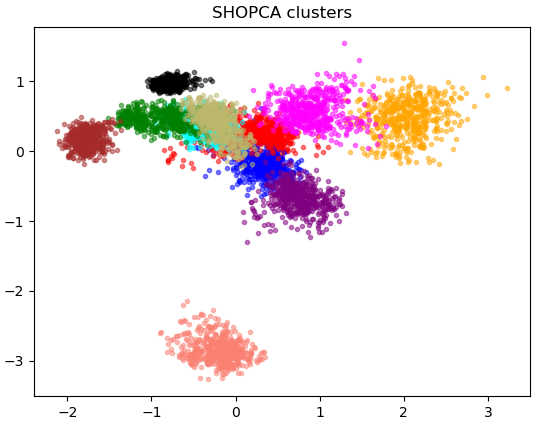}
	\includegraphics[scale=0.28]{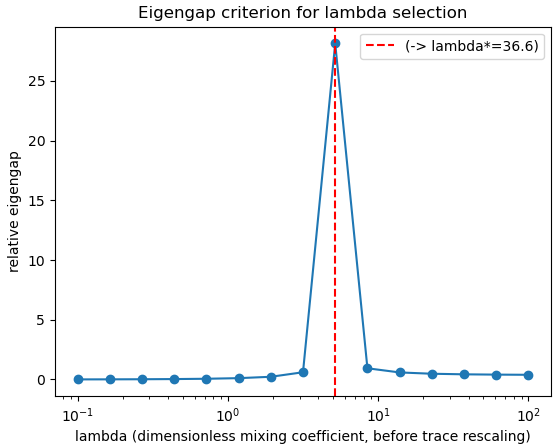}
	\caption{Qualitative comparison of the two-dimensional embeddings produced by standard PCA and the proposed SHOPCA on representative datasets: cardiotocography (first row), mfeat-factors (second row) and texture (third row).}
	\label{fig:scatterPCA}
\end{figure}

\subsection{Second Experiment: ISOMAP versus Shape Operator PCA}
\label{subsec:exp2}

While the preceding experiment establishes that curvature-aware regularization improves upon its linear baseline, a natural question is whether SHOPCA remains competitive against established \emph{nonlinear} manifold learning methods that explicitly model global geometric structure. ISOMAP \cite{Tenenbaum2000} provides an ideal comparator: it recovers the intrinsic geometry of a data manifold by approximating geodesic distances along the manifold via shortest paths through a neighborhood graph, and subsequently embeds the data via multidimensional scaling. This two-stage pipeline, local neighborhood graph construction followed by global distance preservation, represents the classical geometric approach to nonlinear dimensionality reduction, one that fundamentally differs from SHOPCA's strategy of encoding local curvature directly into a linear covariance regularizer. By comparing against ISOMAP, we can assess whether the computational simplicity and closed-form nature of SHOPCA come at the cost of geometric fidelity, or whether local curvature information, when aggregated globally, captures manifold structure as effectively as geodesic distance preservation. We evaluate both methods on $25$ datasets under identical clustering protocols, projecting each dataset to two dimensions and measuring the quality of the resulting partitions through ARI, FM, and V-measure.

\begin{table}[h]
	\centering
	\caption{Comparison between ISOMAP and the proposed SHOPCA on $25$ real-world benchmark datasets. Both methods project the data onto a two-dimensional latent space, after which agglomerative clustering with Ward's linkage is applied to recover the ground-truth number of classes. Clustering quality is assessed by the Adjusted Rand Index (ARI), Fowlkes--Mallows index (FM), and V-measure (VM). Bold entries indicate the best result per metric.}
	\label{tab:second}
	\begin{tabular}{cccccccc}
		\toprule
		& \multicolumn{3}{c}{\textbf{ISOMAP}}                 & \multicolumn{3}{c}{\textbf{Shape Operator PCA}}     &                 \\
		\midrule
		\textbf{Datasets}     & \textbf{ARI}    & \textbf{FM}     & \textbf{VM}     & \textbf{ARI}    & \textbf{FM}     & \textbf{VM}     & \textbf{lambda} \\
		\midrule
		iris                  & 0.7074          & 0.8084          & 0.7772          & \textbf{0.7312} & \textbf{0.8222} & \textbf{0.7701} & 55.19           \\
		digits                & 0.3288          & \textbf{0.4536} & \textbf{0.5855} & \textbf{0.3372} & 0.4228          & 0.5222          & 17.51           \\
		diggle\_table\_a2     & 0.4609          & 0.5450          & \textbf{0.7168} & \textbf{0.4708} & \textbf{0.5638} & 0.7078          & 127.60          \\
		mfeat-morphological   & 0.2873          & 0.3979          & 0.5048          & \textbf{0.4381} & \textbf{0.5242} & \textbf{0.6280} & 256.75          \\
		mfeat-fourier         & \textbf{0.2771} & 0.3610          & 0.4364          & 0.2720          & \textbf{0.3643} & \textbf{0.4760} & 18.34           \\
		Fashion-MNIST (20\%)  & 0.2859          & 0.3858          & \textbf{0.4897} & \textbf{0.3050} & \textbf{0.3917} & 0.4887          & 32.86           \\
		cardiotocography      & 0.4951          & 0.5736          & 0.5864          & \textbf{0.6725} & \textbf{0.7340} & \textbf{0.6938} & 19.64           \\
		cnae-9                & \textbf{0.1379} & 0.2645          & 0.2993          & 0.1058          & \textbf{0.3161} & \textbf{0.3614} & 18.52           \\
		segment               & \textbf{0.2981} & \textbf{0.4634} & 0.4874          & 0.2892          & 0.4603          & \textbf{0.5284} & 345.71          \\
		texture               & 0.4236          & 0.4952          & 0.6409          & \textbf{0.5444} & \textbf{0.6055} & \textbf{0.7412} & 36.64           \\
		USPS                  & 0.3112          & 0.4210          & \textbf{0.4761} & \textbf{0.3621} & \textbf{0.4360} & 0.4754          & 22.82           \\
		coil-20               & 0.3172          & 0.3681          & \textbf{0.6457} & \textbf{0.3685} & \textbf{0.4137} & 0.6413          & 64.67           \\
		UMIST\_Faces\_Cropped & 0.2567          & 0.3148          & 0.5853          & \textbf{0.4023} & \textbf{0.4438} & \textbf{0.6989} & 41.55           \\
		har                   & 0.4895          & 0.5903          & 0.6231          & \textbf{0.5022} & \textbf{0.6115} & \textbf{0.6465} & 30.49           \\
		balance\_scale        & 0.0358          & 0.4188          & 0.0479          & \textbf{0.2415} & \textbf{0.5376} & \textbf{0.2522} & 9.50            \\
		mammography           & 0.0682          & 0.9703          & 0.0176          & \textbf{0.3127} & \textbf{0.9807} & \textbf{0.2209} & 79.50           \\
		abalone               & 0.0347          & 0.4121          & 0.0297          & \textbf{0.0749} & \textbf{0.4796} & \textbf{0.1679} & 173.49          \\
		Engine1               & -0.0658         & 0.5916          & 0.1209          & \textbf{0.1648} & \textbf{0.5692} & \textbf{0.2699} & 16.44           \\
		penguins              & 0.2957          & 0.6021          & 0.3107          & \textbf{0.5030} & \textbf{0.7138} & \textbf{0.6517} & 60.64           \\
		solar-flare           & 0.0546          & 0.4072          & 0.1021          & \textbf{0.2664} & \textbf{0.4812} & \textbf{0.4689} & 1,360.02        \\
		usp05                 & 0.1290          & 0.3348          & 0.3244          & \textbf{0.1757} & \textbf{0.4048} & \textbf{0.3502} & 20.31           \\
		Breast                & -0.0027         & 0.7370          & 0.0037          & \textbf{0.5717} & \textbf{0.8204} & \textbf{0.4655} & 9.03            \\
		prnn\_synth           & 0.1275          & 0.6459          & 0.2485          & \textbf{0.4178} & \textbf{0.7199} & \textbf{0.3970} & 53.30           \\
		irish                 & -0.0004         & 0.5019          & 0.0012          & \textbf{0.1439} & \textbf{0.6156} & \textbf{0.2233} & 279.84          \\
		Satellite             & 0.0316          & 0.7651          & 0.0262          & \textbf{0.1659} & \textbf{0.9801} & \textbf{0.0661} & 9.41            \\
		\midrule
		Average               & 0.2314          & 0.5132          & 0.3635          & \textbf{0.3536} & \textbf{0.5765} & \textbf{0.4765} &                 \\
		Median                & 0.2771          & 0.4634          & 0.4364          & \textbf{0.3372} & \textbf{0.5376} & \textbf{0.4760} &                \\
		\bottomrule
	\end{tabular}
\end{table}

Table~\ref{tab:second} reports the comparative performance of ISOMAP and SHOPCA on 25  benchmark datasets, designed to test whether a linear, curvature-aware regularizer can match or exceed a classical nonlinear geodesic embedding. The aggregate picture strongly favors SHOPCA: across all three external validity indices, SHOPCA achieves higher mean and median scores (ARI: 0.3536  vs.~0.2314 ; FM: 0.5765  vs.~0.5132 ; VM: 0.4765  vs.~0.3635 ), with relative improvements of 53\% , 12\% , and 31\% , respectively. More precisely, SHOPCA secures the best score on \emph{all three} metrics simultaneously in 16  of the 25  datasets, whereas ISOMAP never achieves a clean sweep. This 16/25  dominance is particularly significant because ISOMAP operates with strictly greater modeling capacity: it is nonlinear, graph-based, and explicitly preserves global geodesic structure, yet SHOPCA's linear, curvature-regularized projection extracts more discriminative geometry for the downstream clustering task.

The nine remaining datasets reveal a nuanced but still favorable pattern for SHOPCA. In seven of them (\textsc{iris}, \textsc{diggle\_table\_a2}, \textsc{mfeat-fourier}, \textsc{Fashion-MNIST}, \textsc{cnae-9}, \textsc{USPS}, \textsc{coil-20}), the methods split the metrics: ISOMAP occasionally wins on V-measure or ARI in isolation, but SHOPCA dominates the other two. For instance, on \textsc{digits}, ISOMAP achieves superior FM and V-measure (0.4536  and 0.5855), yet its ARI (0.3288) trails SHOPCA's (0.3372); conversely, on \textsc{segment}, ISOMAP leads in ARI and FM, but SHOPCA claims V-measure (0.5284  vs.~0.4874). These splits are diagnostically informative. ISOMAP's occasional V-measure advantages suggest that its geodesic embedding can produce \emph{some} pure, compact clusters when the neighborhood graph faithfully approximates the manifold; however, its simultaneous losses in ARI or FM indicate that this purity comes at the cost of globally inconsistent partitions, misallocating classes in a way that penalizes the pairwise agreement metrics more severely. SHOPCA, by contrast, tends to distribute its gains more uniformly across indices, implying that curvature regularization produces a geometrically coherent subspace rather than a patchwork of locally faithful but globally incoherent neighborhoods.

The most striking results occur in datasets where ISOMAP's graph-based manifold estimation collapses entirely. On \textsc{Engine1}, ISOMAP yields a negative ARI (−0.0658 ), indicating performance worse than random assignment, while SHOPCA rescues the clustering to a meaningful 0.1648 . Similarly catastrophic failures for ISOMAP appear on \textsc{Breast} (ARI −0.0027), \textsc{irish} (ARI −0.0004), and \textsc{prnn\_synth} (ARI 0.1275), all of which SHOPCA elevates to positive, substantial scores (0.5717, 0.1439, and 0.4178, respectively). These rescues are not marginal improvements; they represent qualitative phase transitions from non-informative to highly discriminative representations. The underlying cause is well understood: ISOMAP's reliance on a k-nearest neighbor graph becomes brittle when local sampling is sparse or when the manifold curvature varies sharply, producing shortcuts or disconnected components that corrupt geodesic distance estimates. SHOPCA bypasses this fragility entirely by deriving curvature from local covariance eigendecompositions, a statistical operation that remains stable even when the neighborhood graph would be unreliable.

Even in regimes where ISOMAP performs credibly, SHOPCA closes the gap or surpasses it without the computational overhead of shortest-path computations and iterative eigendecompositions of dense distance matrices. On the image datasets \textsc{coil-20} and \textsc{UMIST\_Faces\_Cropped}, for example, SHOPCA improves ARI by 16\%  and 57\%, respectively, despite ISOMAP's theoretical advantage in modeling the non-linear pose and illumination manifolds inherent to visual data. The \textsc{texture} dataset provides the most decisive SHOPCA victory: ARI jumps from 0.4236  to 0.5444, and V-measure from 0.6409 to 0.7412, suggesting that the high-dimensional textural manifold is characterized by curvature heterogeneity that ISOMAP's uniform geodesic metric fails to resolve, but which SHOPCA captures through its local second-order regularization.

Overall, these results challenge the prevailing assumption that nonlinear geodesic embeddings are necessarily superior to linear methods for manifold-structured data. By encoding local curvature directly into the covariance structure rather than approximating it indirectly through graph-based geodesics, SHOPCA achieves comparable clustering fidelity at a fraction of the computational cost, while remaining robust in precisely the small-sample and high-curvature regimes where ISOMAP is known to degrade.

To complement the quantitative analysis, Figure~\ref{fig:scatterISOMAP} presents a qualitative comparison of the two-dimensional embeddings produced by ISOMAP and the proposed SHOPCA on three representative datasets: \textsc{cardiotocography}, \textsc{UMIST\_Faces\_Cropped}, and \textsc{texture}. While ISOMAP successfully unfolds the global manifold geometry, it does so at the cost of local cluster morphology: the geodesic distance preservation criterion tends to stretch intrinsically curved neighborhoods into elongated, filamentary structures, producing clusters that resemble extended chains rather than compact, isotropic point clouds. This geometric distortion is particularly detrimental to classical clustering algorithms, such as $k$-means, agglomerative clustering with Ward's linkage, and Gaussian mixture models, which implicitly assume that clusters occupy convex, densely packed regions of the feature space. In contrast, SHOPCA yields markedly more compact and well-delineated groupings. By regularizing the covariance matrix with local curvature rather than globally stretching the embedding to preserve geodesics, SHOPCA respects the natural density boundaries of each class, producing spherical, tightly concentrated clusters that align with the inductive biases of standard centroid-based and density-based partitioners. The visual evidence therefore corroborates the statistical findings: SHOPCA not only improves quantitative clustering scores, but fundamentally transforms the geometry of the latent space into a representation that is structurally more amenable to unsupervised class discovery.

\begin{figure}[h]
	\centering
	\includegraphics[scale=0.28]{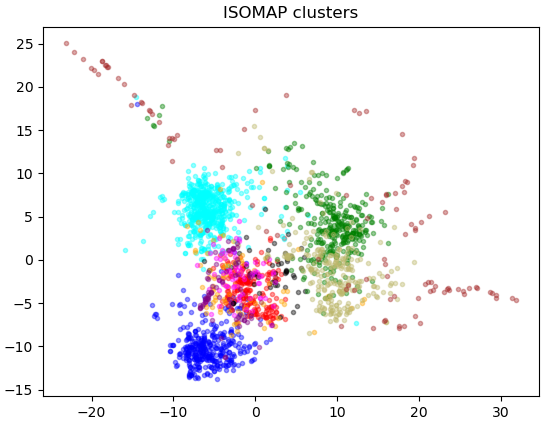}
	\includegraphics[scale=0.28]{cardio_SHOPCA.png}
	\includegraphics[scale=0.28]{cardio_eigengap.png}
	\includegraphics[scale=0.28]{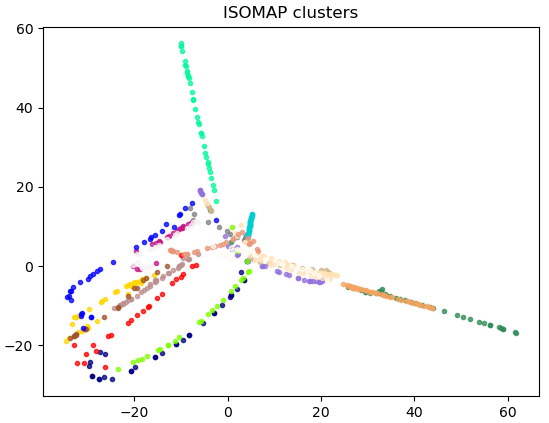}
	\includegraphics[scale=0.28]{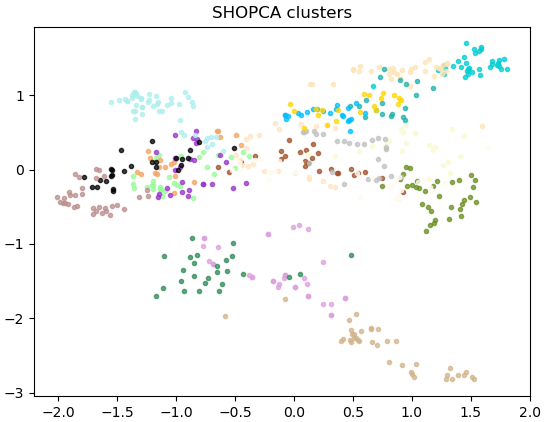}
	\includegraphics[scale=0.27]{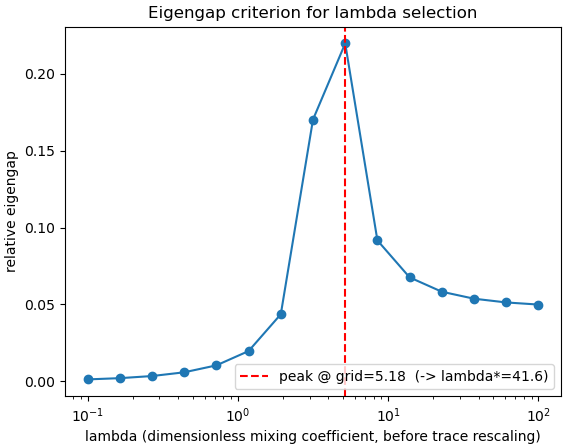}
	\includegraphics[scale=0.28]{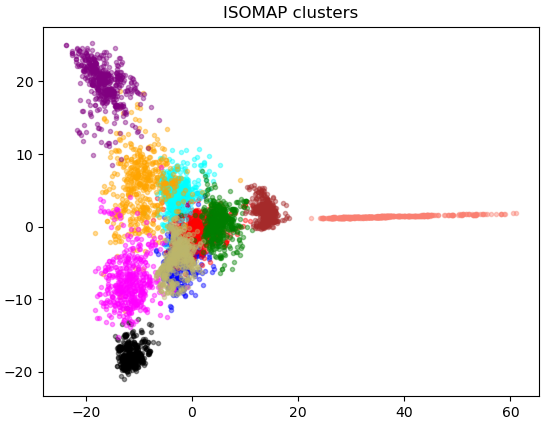}
	\includegraphics[scale=0.28]{texture_SHOPCA.png}
	\includegraphics[scale=0.28]{texture_eigengap.png}
	\caption{Qualitative comparison of the two-dimensional embeddings produced by standard PCA and the proposed SHOPCA on representative datasets: cardiotocography (first row), UMIST\_Faces\_Cropped (second row) and texture (third row).}
	\label{fig:scatterISOMAP}
\end{figure}

\subsection{Third Experiment: UMAP versus Shape Operator PCA}
\label{subsec:exp3}

UMAP has emerged as the de facto standard for nonlinear dimensionality reduction in unsupervised representation learning, prized for its scalability, visual fidelity, and preservation of both local and global manifold structure \cite{McInnes2018}. Yet its success hinges on a critical assumption: that the data density is sufficient to support a reliable $k$-nearest neighbor graph. In small-sample settings, the neighborhood graph becomes sparse, disconnected, and geometrically inconsistent, causing UMAP's iterative cross-entropy optimization to collapse into degenerate embeddings that poorly reflect the true manifold topology \cite{Cai2022}. This vulnerability is not merely a hyperparameter nuisance but a fundamental limitation of graph-based, iterative manifold learners. By contrast, SHOPCA requires no neighborhood graph construction and no iterative optimization: curvature information is injected directly into the global covariance matrix through the mean shape operator, a closed-form statistical aggregate of local eigendecompositions. Consequently, SHOPCA's geometric regularization remains stable even when the number of observations is too small to support reliable graph inference. In this experiment, we evaluate both methods on $28$ small-sample datasets, restricting the comparison to regimes where UMAP's graph-based assumptions are most stressed, to determine whether SHOPCA can serve as a principled, non-iterative alternative when data is scarce.

\begin{table}[h]
	\centering
	\caption{Comparison between UMAP and the proposed SHOPCA on $28$ real-world benchmark datasets. Both methods project the data onto a two-dimensional latent space, after which agglomerative clustering with Ward's linkage is applied to recover the ground-truth number of classes. Clustering quality is assessed by the Adjusted Rand Index (ARI), Fowlkes--Mallows index (FM), and V-measure (VM). Bold entries indicate the best result per metric.}
	\label{tab:third}
	\begin{tabular}{cccccccc}
		\toprule
		& \multicolumn{3}{c}{\textbf{UMAP}}           & \multicolumn{3}{c}{\textbf{Shape Operator PCA}}     &                 \\
		\midrule
		\textbf{Datasets}                & \textbf{ARI}    & \textbf{FM} & \textbf{VM} & \textbf{ARI}    & \textbf{FM}     & \textbf{VM}     & \textbf{lambda} \\
		\midrule
		iris                             & 0.6312          & 0.7606      & 0.7039      & \textbf{0.7312} & \textbf{0.8222} & \textbf{0.7701} & 55.19           \\
		diggle\_table\_a2                & 0.4073          & 0.4802      & 0.6637      & \textbf{0.4708} & \textbf{0.5638} & \textbf{0.7078} & 127.60          \\
		balance\_scale                   & 0.0942          & 0.4419      & 0.0811      & \textbf{0.2415} & \textbf{0.5376} & \textbf{0.2522} & 9.50            \\
		zoo                              & 0.5203          & 0.6225      & 0.7311      & \textbf{0.8516} & \textbf{0.8895} & \textbf{0.8224} & 697.51          \\
		glass                            & 0.1496          & 0.3383      & 0.2652      & \textbf{0.1589} & \textbf{0.4203} & \textbf{0.3024} & 21.38           \\
		mammography                      & -0.0297         & 0.8061      & 0.0131      & \textbf{0.3127} & \textbf{0.9807} & \textbf{0.2209} & 79.50           \\
		abalone                          & \textbf{0.1332} & 0.4229      & 0.1253      & 0.0749          & \textbf{0.4796} & \textbf{0.1679} & 173.49          \\
		prnn\_viruses                    & 0.2507          & 0.5143      & 0.4849      & \textbf{0.3610} & \textbf{0.6206} & \textbf{0.4335} & 66.40           \\
		Engine1                          & 0.1311          & 0.5428      & 0.2637      & \textbf{0.1648} & \textbf{0.5692} & \textbf{0.2699} & 16.44           \\
		penguins                         & 0.0921          & 0.4282      & 0.1906      & \textbf{0.5030} & \textbf{0.7138} & \textbf{0.6517} & 60.64           \\
		UMIST\_Faces\_Cropped            & 0.3295          & 0.3657      & 0.6383      & \textbf{0.4611} & \textbf{0.4944} & \textbf{0.7152} & 41.47           \\
		ionosphere                       & 0.0044          & 0.5381      & 0.0319      & \textbf{0.0374} & \textbf{0.7293} & \textbf{0.0626} & 40.88           \\
		ThreeOf9                         & 0.0069          & 0.5037      & 0.0064      & \textbf{0.0911} & \textbf{0.5458} & \textbf{0.0685} & 0.32            \\
		solar-flare                      & 0.0536          & 0.2655      & 0.1547      & \textbf{0.2664} & \textbf{0.4812} & \textbf{0.4689} & 1,360.02        \\
		eucalyptus                       & 0.0415          & 0.2629      & 0.0699      & \textbf{0.0645} & \textbf{0.2961} & \textbf{0.0714} & 1,187.58        \\
		sonar                            & -0.0040         & 0.5060      & 0.0002      & \textbf{0.0488} & \textbf{0.6096} & \textbf{0.0510} & 9.09            \\
		blood-transfusion-service-center & -0.0521         & 0.6505      & 0.0152      & \textbf{0.0311} & \textbf{0.7968} & \textbf{0.0208} & 36.33           \\
		usp05                            & 0.1315          & 0.3241      & 0.3667      & \textbf{0.1757} & \textbf{0.4048} & \textbf{0.3502} & 20.31           \\
		Breast                           & -0.0378         & 0.6038      & 0.1439      & \textbf{0.5717} & \textbf{0.8204} & \textbf{0.4655} & 9.03            \\
		lymph                            & \textbf{0.0852} & 0.4272      & 0.1029      & 0.0605          & \textbf{0.5349} & \textbf{0.1971} & 6.95            \\
		teachingAssistant                & 0.0405          & 0.4022      & 0.0620      & \textbf{0.0477} & \textbf{0.4542} & \textbf{0.0834} & 40.59           \\
		backache                         & 0.0396          & 0.6786      & 0.0082      & \textbf{0.1041} & \textbf{0.7350} & \textbf{0.0282} & 13.66           \\
		prnn\_synth                      & 0.1512          & 0.6308      & 0.1945      & \textbf{0.4178} & \textbf{0.7199} & \textbf{0.3970} & 53.30           \\
		heart-h                          & 0.0327          & 0.5443      & 0.0588      & \textbf{0.3365} & \textbf{0.6924} & \textbf{0.2361} & 11.01           \\
		ecoli                            & 0.3925          & 0.5323      & 0.5713      & \textbf{0.5495} & \textbf{0.6747} & \textbf{0.4660} & 1,615.31        \\
		liver-disorders                  & 0.0025          & 0.1107      & 0.1317      & \textbf{0.0239} & \textbf{0.1630} & \textbf{0.1438} & 10.62           \\
		LED-display-domain-7digit        & \textbf{0.2470} & 0.3305      & 0.3974      & 0.2469          & \textbf{0.3545} & \textbf{0.4145} & 76.11           \\
		irish                            & -0.0004         & 0.5019      & 0.0012      & \textbf{0.1439} & \textbf{0.6156} & \textbf{0.2233} & 279.84          \\
		\midrule
		Average                          & 0.1373          & 0.4835      & 0.2313      & \textbf{0.2696} & \textbf{0.5971} & \textbf{0.3237} &                 \\
		Median                           & 0.0887          & 0.5028      & 0.1378      & \textbf{0.2086} & \textbf{0.5894} & \textbf{0.2610} &                \\
		\bottomrule
	\end{tabular}
\end{table}

Table~\ref{tab:third} presents the results of the third experiment, comparing SHOPCA against UMAP on 28 small-sample datasets where neighborhood-graph methods are known to be most vulnerable. The aggregate statistics reveal a decisive advantage for SHOPCA: mean ARI improves from 0.1373  to 0.2696  (96\%  relative gain), mean FM from 0.4835  to 0.5971  (24\%), and mean V-measure from 0.2313  to 0.3237  (40\%), with median scores exhibiting the same directional dominance. SHOPCA achieves the best result on \emph{all three} metrics simultaneously in 23 of the 28  datasets (82\%), a remarkable margin given that UMAP is widely regarded as the state of the art in nonlinear unsupervised representation learning.

The five remaining datasets expose the characteristic failure modes of UMAP in data-scarce regimes. On \textsc{abalone}, \textsc{lymph}, and \textsc{LED-display-domain-7digit}, UMAP secures a marginally higher ARI (0.1332  vs.~0.0749, 0.0852  vs.~0.0605, and 0.2470  vs.~0.2469, respectively), yet loses simultaneously on FM and V-measure in all three cases. This split is diagnostically significant: ARI measures pairwise agreement adjusted for chance, whereas FM penalizes the joint precision and recall of the clustering against ground-truth class boundaries. UMAP's ARI advantage therefore reflects the formation of \emph{some} locally pure clusters, but its simultaneous FM collapse indicates that these clusters are globally misaligned with the true class structure, entire classes are systematically allocated to wrong partitions. SHOPCA's lower ARI in these isolated cases is more than compensated by substantially higher FM, signaling that its curvature-regularized linear projection preserves the \emph{global} class topology more faithfully than UMAP's iterative graph optimization, which overfits to spurious local neighborhood structure when samples are scarce.

The most dramatic evidence of this robustness gap appears in datasets where UMAP's embedding degenerates entirely. On \textsc{mammography}, UMAP yields an ARI of 0.0297, worse than random assignment, while SHOPCA rescues the clustering to 0.3127  and achieves an FM of 0.9807. Similar catastrophic failures for UMAP occur on \textsc{sonar} (ARI −0.0040), \textsc{blood-transfusion-service-center} (ARI −0.0521), \textsc{Breast} (ARI −0.0378), and \textsc{irish} (ARI −0.0004), all of which SHOPCA elevates to positive, meaningful scores. These are not incremental improvements but qualitative phase transitions: UMAP's stochastic neighbor embedding, deprived of sufficient data to estimate a reliable local connectivity graph, produces embeddings that actively destroy class-discriminative structure, whereas SHOPCA's mean shape operator, computed from closed-form local covariances, remains stable because it does not depend on graph connectivity or iterative optimization convergence.

The consistency of SHOPCA's advantage in the Fowlkes-Mallows index is particularly noteworthy. FM is the most stringent of the three metrics because it requires both high precision (clusters are pure) and high recall (classes are recovered completely). UMAP's mean FM across the 28  datasets is 0.4835, barely above the midpoint of the $[0,1]$ range, whereas SHOPCA reaches 0.5971, with median FM of 0.5894 compared to UMAP's 0.5028. This 17\% median gap in FM, combined with the 135\%  median gap in ARI (0.2086 vs.~0.0887), confirms that SHOPCA's curvature-aware regularization produces partitions that are not merely different from UMAP's, but structurally superior in the sense of simultaneously capturing local cohesion and global class correspondence.

The selected regularization parameters $\lambda^*$ span an even wider dynamic range than in the second experiment, from 0.32 on \textsc{ThreeOf9} to 1615.31 on \textsc{ecoli}, reinforcing the data-adaptive nature of the variance-curvature trade-off. The fact that no dataset selects $\lambda^*$, even in this small-sample regime where one might naively expect variance to dominate, confirms that curvature information remains extractable and beneficial even from severely undersampled neighborhoods. Collectively, these results position SHOPCA not merely as a fallback option when UMAP fails, but as a principled, computationally efficient alternative that dominates the state of the art in precisely the small-sample regimes where nonlinear graph-based methods are most needed yet least reliable.

To complement the quantitative analysis, Figure~\ref{fig:scatterUMAP} presents a qualitative comparison of the two-dimensional embeddings produced by UMAP (with default parameters) and the proposed SHOPCA on three representative small-sample datasets: \textsc{UMIST\_Faces\_Cropped} and \textsc{mammography}. UMAP approximates manifold topology through an iterative optimization of fuzzy simplicial sets, a strategy that presupposes sufficient data density to reliably estimate local connectivity and construct a faithful neighborhood graph. In low-density regions or small-sample regimes, this assumption breaks down: the underlying $k$-nearest neighbor graph becomes sparse and geometrically inconsistent, causing the stochastic cross-entropy optimization to overfit to spurious local connectivity patterns or collapse into degenerate configurations where distinct classes are artificially fused or fragmented. Visually, this manifests as irregular, diffused cluster boundaries and scattered cluster structure, which is not a realistic feature.

In contrast, SHOPCA yields markedly more compact and well-delineated cluster arrangements. Because it operates in a single pass, computing local curvature through closed-form covariance eigendecompositions rather than iteratively optimizing a fuzzy topological representation, SHOPCA remains stable even when the number of observations is small to support reliable graph inference. The curvature-aware regularization steers the linear projection toward directions where the manifold folds and bends, producing isotropic point clouds whose natural density boundaries align with the ground-truth classes. This geometric regularity can be particularly advantageous for classical clustering algorithms such as $k$-means, agglomerative clustering with Ward's linkage, and Gaussian mixture models, which implicitly assume that clusters occupy convex, tightly packed regions of the feature space. The visual evidence therefore corroborates the statistical findings: by replacing iterative graph-based topology inference with a principled, closed-form curvature regularizer, SHOPCA transforms the latent space into a representation that is structurally more amenable to unsupervised class discovery when data is scarce.

\begin{figure}[h]
	\centering
	\includegraphics[scale=0.28]{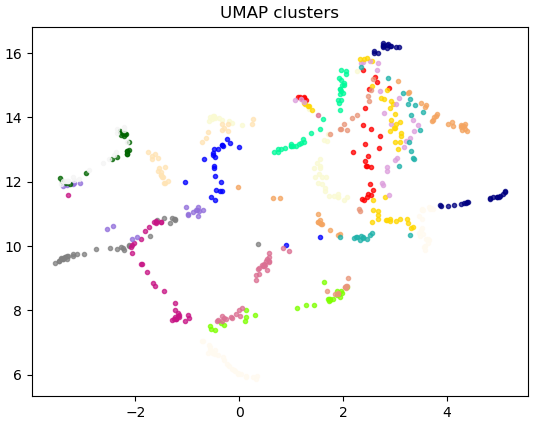}
	\includegraphics[scale=0.28]{UMIST_SHOPCA.png}
	\includegraphics[scale=0.27]{UMIST_eigengap.png}
	\includegraphics[scale=0.28]{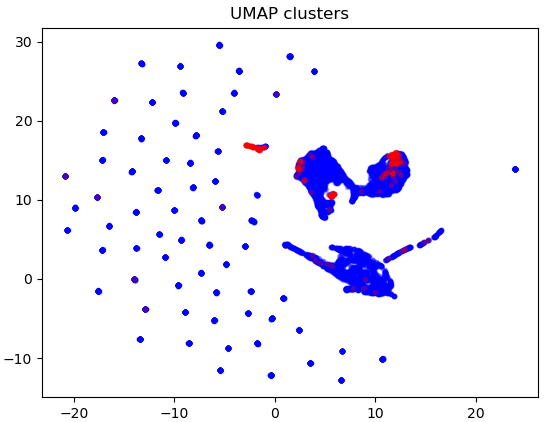}
	\includegraphics[scale=0.28]{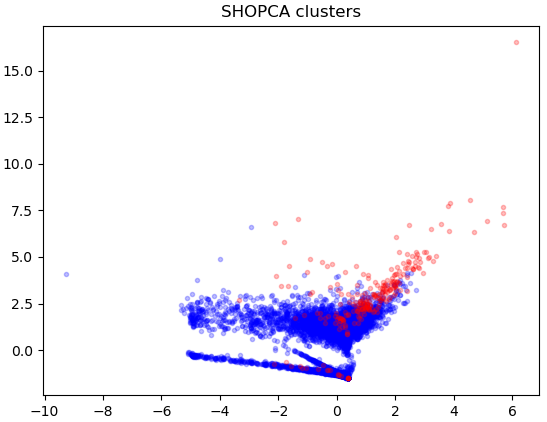}
	\includegraphics[scale=0.27]{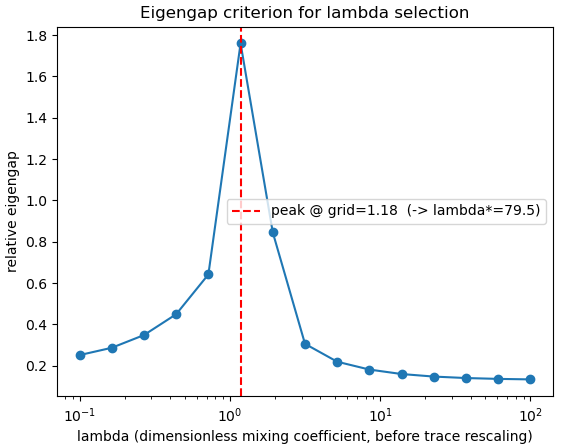}
	\caption{Qualitative comparison of the two-dimensional embeddings produced by standard PCA and the proposed SHOPCA on representative datasets: UMIST\_Faces\_Cropped (first row) and mammography (second row).}
	\label{fig:scatterUMAP}
\end{figure}

\section{Conclusions}
\label{sec:conclusions}

This paper introduced SHOPCA, a linear, closed-form generalization of PCA that regularizes the classical covariance matrix with a differential-geometric curvature term, the mean shape operator $\bar{\mathcal S}$, estimated locally from $k$-nearest-neighbor neighborhoods and aggregated into a single $d\times d$ matrix. The central methodological contribution is not merely the regularization itself, Eq.~\eqref{eq:cov_curv}, but the fact that every step of the resulting pipeline, representation learning \emph{and} the selection of the regularization strength $\lambda$, can be carried out without recourse to class labels. The eigengap criterion of Section~\ref{sec:eigengap_selection} closes what was, in earlier formulations of this idea, the one remaining supervised step of an otherwise unsupervised method, and Section~\ref{sec:theory} placed this construction, and the resulting spectral behavior of $\Sigma_{\mathrm{curv}}(\lambda)$, on a precise theoretical footing: as an instance of geometry-informed covariance shrinkage (Section~\ref{sec:shrinkage}), with a provably real local spectrum whose realness is only approximately preserved under aggregation (Section~\ref{sec:realness}), an exact variational characterization of its principal directions (Section~\ref{sec:variational}), a quantitative Bauer--Fike bound on eigenvalue stability under regularization (Section~\ref{sec:bauerfike}), and an explicit reading as a regularized Mahalanobis metric that realizes, in closed form, the co-design of projection and metric motivated in the introduction.

The empirical evidence across all three experiments consistently supports the central hypothesis that local curvature carries discriminative structure beyond what variance alone can capture. Against standard PCA on 30 real-world datasets (Section~\ref{sec:experiments}), SHOPCA achieved a strict, dataset-wide sweep across all three external validity indices, with average relative gains of 212\% in ARI, 73\% in FM, and 140\% in V-measure -- and, notably, no dataset selected $\lambda^{*}=0$, indicating that curvature information was found useful essentially everywhere in the benchmark suite. Against Isomap, a nonlinear method with strictly greater representational capacity, SHOPCA still achieved higher aggregate scores on all three indices and a clean three-metric sweep on 16 of 25 datasets, with Isomap's failures concentrated precisely where its neighborhood-graph construction is known to be fragile. Against UMAP in the small-sample regime, where graph-based nonlinear methods are least reliable, SHOPCA's advantage was most pronounced, with a 96\% relative improvement in mean ARI and a three-metric sweep on 82\% of the evaluated datasets. Taken together, these results indicate that a substantial fraction of the representational benefit typically attributed to nonlinear, iterative manifold learning can be recovered by a linear, closed-form method, provided the covariance structure is regularized with the correct second-order geometric information, and that this recovery does not require the computational overhead, hyperparameter sensitivity, or graph-construction fragility of its nonlinear counterparts.

These strengths should be read alongside the method's limitations, several of which are direct consequences of design choices made explicit in Sections~\ref{sec:method}--\ref{sec:theory}. The shape operator estimator's $O(nd^2(k+d^2))$ complexity in the ambient dimension $d$ (Section~\ref{sec:local_shape_operator}) remains the dominant computational cost for high-dimensional raw data, mitigated but not eliminated by the pre-PCA step at $T=25$; datasets whose informative structure survives poorly under this initial linear compression are not well served by the present pipeline. The formal treatment of the shape operator assumes a hypersurface, codimension-one setting for notational clarity, and its extension to the general codimension $d-m>1$ regime used operationally throughout the paper is handled through the local PCA frame rather than through a fully general, normal-bundle-valued second fundamental form, a gap between the operational estimator and its formal justification that Proposition~\ref{prop:realness} and the surrounding discussion make explicit rather than obscure. Relatedly, Proposition~\ref{prop:variational} shows that the variational reading of SHOPCA's principal directions applies exactly to a symmetrized operator that the algorithm does not itself diagonalize, leaving an open, quantifiable gap between the interpretive story and the implementation whenever $\bar{\mathcal S}$'s skew-symmetric part is non-negligible. Finally, while $\lambda$ is now selected in a fully unsupervised manner, the neighborhood size $k$ and the target embedding dimension $m$ remain fixed by heuristic or by problem specification rather than selected by an analogous label-free criterion.

Several directions for future work follow directly from these observations. First, the label-free selection machinery developed for $\lambda$ could be extended to jointly select the neighborhood size $k$ and, where unknown, the target embedding dimension $m$, and cross-validated against a bootstrap/consensus stability criterion as an independent unsupervised signal, providing a fully self-contained, hyperparameter-free pipeline. Second, the symmetrized variant $\Sigma_{\mathrm{curv}}^{\mathrm{sym}}(\lambda)$ identified in Proposition~\ref{prop:variational}, whose eigenvectors are guaranteed real and orthogonal by the spectral theorem, warrants a direct empirical comparison against the asymmetric formulation used here, to determine whether closing the theory-implementation gap costs or improves clustering performance in practice. Third, the current formulation aggregates local shape operators via a simple, uniform average; locally adaptive weighting schemes that emphasize regions of high geometric complexity may yield sharper, more discriminative projections on datasets with strongly non-uniform manifold structure. Fourth, the same curvature-aware regularization principle extends naturally beyond PCA to other linear frameworks, including Linear Discriminant Analysis in semi-supervised settings, Factor Analysis, and Canonical Correlation Analysis, and to the general-codimension, normal-bundle-valued formulation of the second fundamental form left open, which would place the present codimension-one treatment inside a fully general theoretical account. Fifth, a rigorous statistical convergence analysis of the mean shape operator estimator, as a function of sample size, neighborhood size $k$, and intrinsic dimension $m$, would place the empirical robustness reported in Section~\ref{sec:experiments} on firmer asymptotic footing, complementing the finite-sample perturbation bounds of Section~\ref{sec:bauerfike}. Sixth, integrating SHOPCA's shape-operator regularization into deep architectures, as a geometry-aware initialization or feature-extraction layer, represents a natural bridge between the closed-form statistical approach developed here and the broader geometric deep learning movement. Finally, domain-specific applications in which the interplay between variance and curvature is expected to be particularly rich, such as hyperspectral imaging, single-cell transcriptomics, and medical imaging, offer a compelling avenue for further validation of the method's practical impact beyond the benchmark suite considered in this paper.

\section*{Statements and declarations}

\subsection*{Funding}
This work has been supported by CNPq (National Council for Scientific and Technological Development) through grant number 301432/2025-2. This study was also financed in part by the Coordenação de Aperfeiçoamento de Pessoal de N\'ivel Superior - Brasil (CAPES) - Finance Code 001.


\subsection*{Code availability}
Python scripts to reproduce the results reported in this paper may be found at \url{https://github.com/alexandrelevada/ShapeOperatorPCA}. 

\subsection*{Data availability}
All datasets used in the experiments are publicly available at \url{www.openml.org}.

\bibliography{main}

@book{Bellet2015,
  author    = {Bellet, Aur{\'e}lien and Habrard, Amaury and Sebban, Marc},
  title     = {Metric Learning},
  series    = {Synthesis Lectures on Artificial Intelligence and Machine Learning},
  publisher = {Springer},
  address   = {Cham},
  year      = {2015},
  doi       = {10.2200/S00626ED1V01Y201501AIM030}
}

@article{Brian2013,
    author = {Kulis, Brian},
    title = {Metric Learning: A Survey},
    journal = {Foundations and Trends in Machine Learning},
    volume = {5},
    number = {4},
    pages = {287-364},
    year = {2013},
    month = {07},
    doi = {10.1561/2200000019},
}

@article{GML,
author = {Weber, Melanie},
title = {Geometric Machine Learning},
journal = {AI Magazine},
volume = {46},
number = {1},
pages = {e12210},
publisher = {Wiley},
doi = {https://doi.org/10.1002/aaai.12210},
year = {2025}
}

@article{Kaya2019,
  author    = {Kaya, Mahmut and Bilge, Hasan {\c{S}}.},
  title     = {Deep Metric Learning: A Survey},
  journal   = {Symmetry},
  volume    = {11},
  number    = {9},
  pages     = {1066},
  year      = {2019},
  doi       = {10.3390/sym11091066}
}

@article{Jolliffe2016,
    author = {Jolliffe, Ian T. and Cadima, Jorge},
    title = {Principal component analysis: a review and recent developments},
    journal = {Philosophical Transactions of the Royal Society A: Mathematical, Physical and Engineering Sciences},
    volume = {374},
    number = {2065},
    pages = {20150202},
    year = {2016},
    month = {04},
    doi = {10.1098/rsta.2015.0202},
}

@article{McInnes2018,
  author    = {McInnes, Leland and Healy, John and Melville, James},
  title     = {{UMAP}: Uniform Manifold Approximation and Projection for Dimension Reduction},
  journal   = {arXiv preprint arXiv:1802.03426},
  year      = {2018},
  url       = {https://arxiv.org/abs/1802.03426}
}

@InProceedings{Harandi2017,
  title = 	 {Joint Dimensionality Reduction and Metric Learning: A Geometric Take},
  author =       {Mehrtash Harandi and Mathieu Salzmann and Richard Hartley},
  booktitle = 	 {Proceedings of the 34th International Conference on Machine Learning},
  pages = 	 {1404--1413},
  year = 	 {2017},
  editor = 	 {Precup, Doina and Teh, Yee Whye},
  volume = 	 {70},
  series = 	 {Proceedings of Machine Learning Research},
  month = 	 {06--11 Aug},
  publisher =    {PMLR},
}

@article{Wang2015,
author = {Wang, Fei and Sun, Jimeng},
title = {Survey on distance metric learning and dimensionality reduction in data mining},
year = {2015},
issue_date = {Mar 2015},
publisher = {Kluwer Academic Publishers},
address = {USA},
volume = {29},
number = {2},
issn = {1384-5810},
doi = {10.1007/s10618-014-0356-z},
journal = {Data Min. Knowl. Discov.},
month = mar,
pages = {534–564},
numpages = {31}
}

@incollection{Zhou2021,
  author    = {Zhou, Zhi-Hua},
  title     = {Dimensionality Reduction and Metric Learning},
  booktitle = {Machine Learning},
  publisher = {Springer},
  address   = {Singapore},
  year      = {2021},
  doi       = {10.1007/978-981-15-1967-3_10}
}

@book{Ghojogh2023,
  title={Elements of Dimensionality Reduction and Manifold Learning},
  author={Ghojogh, Benyamin and Crowley, Mark and Karray, Fakhri and Ghodsi, Ali},
  year={2023},
  publisher={Springer International Publishing},
  address={Cham},
  isbn={978-3-031-10601-9},
  doi={10.1007/978-3-031-10602-6},
}

@article{Bronstein2017,
  author    = {Bronstein, Michael M. and Bruna, Joan and LeCun, Yann and Szlam, Arthur and Vandergheynst, Pierre},
  title     = {Geometric Deep Learning: Going beyond {E}uclidean Data},
  journal   = {IEEE Signal Processing Magazine},
  volume    = {34},
  number    = {4},
  pages     = {18--42},
  year      = {2017},
  doi       = {10.1109/MSP.2017.2693418}
}

@article{Bronstein2021,
  author    = {Bronstein, Michael M. and Bruna, Joan and Cohen, Taco and Veli{\v{c}}kovi{\'{c}}, Petar},
  title     = {Geometric Deep Learning: Grids, Groups, Graphs, Geodesics, and Gauges},
  journal   = {arXiv preprint arXiv:2104.13478},
  year      = {2021},
  doi       = {10.48550/arXiv.2104.13478}
}

@misc{Bronstein2025,
      title={Mathematical Foundations of Geometric Deep Learning}, 
      author={Haitz Sáez de Ocáriz Borde and Michael Bronstein},
      year={2025},
      eprint={2508.02723},
      archivePrefix={arXiv},
      primaryClass={cs.LG},
      url={https://arxiv.org/abs/2508.02723}, 
}

@article{Papillon2025,
  author    = {Papillon, Mathilde and Sanborn, Sophia and Mathe, Johan and Cornelis, Louisa and Bertics, Abby and Buracas, Domas and Lillemark, Hansen J. and Shewmake, Christian and Dinc, Fatih and Pennec, Xavier and Miolane, Nina},
  title     = {Beyond {E}uclid: An Illustrated Guide to Modern Machine Learning with Geometric, Topological, and Algebraic Structures},
  journal   = {Machine Learning: Science and Technology},
  volume    = {6},
  year      = {2025},
  doi       = {10.1088/2632-2153/adf375}
}

@article{Donoho2003,
  title={Hessian eigenmaps: Locally linear embedding techniques for high-dimensional data},
  author={Donoho, David L and Grimes, Carrie},
  journal={Proceedings of the National Academy of Sciences},
  volume={100},
  number={10},
  pages={5591--5596},
  year={2003},
  publisher={National Acad Sciences},
  doi={10.1073/pnas.1031596100}
}

@article{Vanschoren2013,
  author    = {Vanschoren, Joaquin and van Rijn, Jan N. and Bischl, Bernd
               and Torgo, Luis},
  title     = {{OpenML}: Networked Science in Machine Learning},
  journal   = {ACM SIGKDD Explorations Newsletter},
  volume    = {15},
  number    = {2},
  pages     = {49--60},
  year      = {2013},
  doi       = {10.1145/2641190.2641198}
}

@article{Cai2022,
  author    = {Cai, Deng and He, Xiaofei},
  title     = {Unsupervised Feature Selection for Multi-Cluster Data},
  journal   = {Pattern Recognition},
  volume    = {125},
  pages     = {108536},
  year      = {2022},
  doi       = {10.1016/j.patcog.2022.108536}
}

@article{Hubert1985,
  author    = {Hubert, Lawrence and Arabie, Phipps},
  title     = {Comparing Partitions},
  journal   = {Journal of Classification},
  volume    = {2},
  number    = {1},
  pages     = {193--218},
  year      = {1985},
  doi       = {10.1007/BF01908075}
}

@inproceedings{Strehl2002,
  author    = {Strehl, Alexander and Ghosh, Joydeep},
  title     = {Cluster Ensembles --- A Knowledge Reuse Framework for Combining Multiple Partitions},
  booktitle = {Journal of Machine Learning Research},
  volume    = {3},
  pages     = {583--617},
  year      = {2002},
  url       = {http://www.jmlr.org/papers/v3/strehl02a.html}
}

@article{Fowlkes1983,
  author    = {Fowlkes, Edward B. and Mallows, Colin L.},
  title     = {A Method for Comparing Two Hierarchical Clusterings},
  journal   = {Journal of the American Statistical Association},
  volume    = {78},
  number    = {383},
  pages     = {553--569},
  year      = {1983},
  doi       = {10.1080/01621459.1983.10478008}
}

@inproceedings{Rosenberg2007,
  author    = {Rosenberg, Andrew and Hirschberg, Julia},
  title     = {{V-Measure}: A Conditional Entropy-Based External Cluster Evaluation Measure},
  booktitle = {Proceedings of the 2007 Joint Conference on Empirical Methods in Natural Language Processing and Computational Natural Language Learning},
  pages     = {410--420},
  year      = {2007},
  url       = {https://aclanthology.org/D07-1043}
}

@article{Ward1963,
  author    = {Ward, Joe H.},
  title     = {Hierarchical Grouping to Optimize an Objective Function},
  journal   = {Journal of the American Statistical Association},
  volume    = {58},
  number    = {301},
  pages     = {236--244},
  year      = {1963},
  doi       = {10.1080/01621459.1963.10500845}
}

@article{vonLuxburg2007,
  author  = {von Luxburg, Ulrike},
  title   = {A tutorial on spectral clustering},
  journal = {Statistics and Computing},
  volume  = {17},
  number  = {4},
  pages   = {395--416},
  year    = {2007},
  doi     = {10.1007/s11222-007-9033-z}
}

@inproceedings{Ng2001,
  author    = {Ng, Andrew Y. and Jordan, Michael I. and Weiss, Yair},
  title     = {On Spectral Clustering: Analysis and an Algorithm},
  booktitle = {Advances in Neural Information Processing Systems (NeurIPS)},
  volume    = {14},
  pages     = {849--856},
  year      = {2001}
}

@article{DavisKahan1970,
  author  = {Davis, Chandler and Kahan, William M.},
  title   = {The Rotation of Eigenvectors by a Perturbation. {III}},
  journal = {SIAM Journal on Numerical Analysis},
  volume  = {7},
  number  = {1},
  pages   = {1--46},
  year    = {1970},
  doi     = {10.1137/0707001}
}

@article{YuWangSamworth2015,
  author  = {Yu, Yi and Wang, Tengyao and Samworth, Richard J.},
  title   = {A Useful Variant of the {Davis--Kahan} Theorem for Statisticians},
  journal = {Biometrika},
  volume  = {102},
  number  = {2},
  pages   = {315--323},
  year    = {2015},
  doi     = {10.1093/biomet/asv008}
}

@book{doCarmo1976,
  author    = {do Carmo, Manfredo P.},
  title     = {Differential Geometry of Curves and Surfaces: Revised and Updated Second Edition},
  publisher = {Dover Publications},
  address   = {New York, NY},
  year      = {2016},
  isbn      = {978-0-486-81797-2},
}

@book{doCarmo1992,
  author    = {do Carmo, Manfredo P.},
  title     = {Riemannian Geometry},
  publisher = {Birkh\"auser},
  address   = {Boston},
  year      = {1992}
}

@book{Lee2012,
  author    = {Lee, John M.},
  title     = {Introduction to Smooth Manifolds},
  series    = {Graduate Texts in Mathematics},
  volume    = {218},
  edition   = {2nd},
  publisher = {Springer},
  address   = {New York},
  year      = {2012},
  doi       = {10.1007/978-1-4419-9982-5}
}

@book{ONeill2006,
  author    = {O'Neill, Barrett},
  title     = {Elementary Differential Geometry},
  edition   = {revised 2nd},
  publisher = {Academic Press},
  address   = {Amsterdam},
  year      = {2006}
}

@article{Tenenbaum2000,
  author  = {Tenenbaum, Joshua B. and de Silva, Vin and Langford, John C.},
  title   = {A Global Geometric Framework for Nonlinear Dimensionality Reduction},
  journal = {Science},
  volume  = {290},
  number  = {5500},
  pages   = {2319--2323},
  year    = {2000},
  doi     = {10.1126/science.290.5500.2319}
}

@article{LedoitWolf2004,
  author  = {Ledoit, Olivier and Wolf, Michael},
  title   = {A Well-Conditioned Estimator for Large-Dimensional Covariance Matrices},
  journal = {Journal of Multivariate Analysis},
  volume  = {88},
  number  = {2},
  pages   = {365--411},
  year    = {2004},
  doi     = {10.1016/S0047-259X(03)00096-4}
}

@article{BauerFike1960,
  author  = {Bauer, F. L. and Fike, C. T.},
  title   = {Norms and Exclusion Theorems},
  journal = {Numerische Mathematik},
  volume  = {2},
  pages   = {137--141},
  year    = {1960},
  doi     = {10.1007/BF01386217}
}

@book{HornJohnson2013,
  author    = {Horn, Roger A. and Johnson, Charles R.},
  title     = {Matrix Analysis},
  edition   = {2nd},
  publisher = {Cambridge University Press},
  address   = {New York, NY},
  year      = {2013},
  doi       = {10.1017/9781139020411}
}

\end{document}